\documentclass{article}

 \usepackage[main, final]{neurips_2026}

\usepackage[utf8]{inputenc} 
\usepackage[T1]{fontenc}    
\usepackage{hyperref}       
\usepackage{url}            
\usepackage{booktabs}       
\usepackage{amsfonts}       
\usepackage{nicefrac}       
\usepackage{microtype}      
\usepackage{xcolor}         
\usepackage{float}
\usepackage{amssymb}            
\usepackage{mathtools}          
\usepackage{mathrsfs}           
\usepackage{amsthm} 
\usepackage{array}
\usepackage{tabularx}
\newtheorem{assumption}{Assumption}          
\newtheorem{theorem}{Theorem}
\usepackage{graphicx}           
\usepackage{subcaption}         
\usepackage[space]{grffile}     
\usepackage{url}                
\theoremstyle{definition} 
\newtheorem{definition}{Definition}
\usepackage{lipsum}  
\usepackage{balance} 
\DeclareMathOperator*{\argmax}{arg\,max}
\usepackage{subcaption}
\usepackage{bm}
\usepackage{booktabs}
\usepackage[table]{xcolor}
\usepackage{pgfplots}
\pgfplotsset{compat=1.18}
\usepackage{colortbl}

\newcounter{rq}
\renewcommand{\therq}{RQ\arabic{rq}}
\usepackage{wrapfig}
\usepackage{enumitem}

\title{NestRL: A Nested Training Regime for Mutual Adaptation in Human-AI Teaming}

\author{%
  Upasana Biswas\textsuperscript{1} \quad
  Durgesh Kalwar\textsuperscript{1} \quad
  Subbarao Kambhampati\textsuperscript{1} \quad
  Sarath Sreedharan\textsuperscript{2} \\
  \textsuperscript{1}School of Computing and AI, Arizona State University, Tempe, AZ, USA \\
  \textsuperscript{2}Department of Computer Science, Colorado State University, Fort Collins, CO, USA \\
}

\begin{document}

\maketitle

\begin{abstract}
  Mutual adaptation is a central challenge in human-AI teaming, as humans naturally adjust their strategies in response to an AI agent's behavior. Existing approaches attempt to approximate human behavior by diversifying training partners; however, these partners are typically static and fail to capture the adaptive nature of human teammates. When agents are trained jointly in standard multi-agent settings, they often converge to opaque coordination strategies that work only with their co-trained partners, leading to poor generalization. To model adaptive human behavior, we formulate human-AI teaming as an Interactive Partially Observable Markov Decision Process (I-POMDP). We propose \emph{NestRL}, a nested training regime that learns finitely nested I-POMDP reasoning by training agents at each level against adaptive agents from the level below. This exposes agents to adaptive behavior while preventing emergence of opaque coordination strategies. We provide theoretical analysis showing that NestRL agents avoid convergence to partner-specific strategies, and validate this empirically in the Overcooked domain against state-of-the-art baselines. NestRL achieves higher task performance with both unseen adaptive agents and real human teammates, while exhibiting significantly greater adaptability over the course of interaction.
\end{abstract}
\section{Introduction}
Creating reinforcement learning (RL) agents that effectively collaborate with unseen partners i.e. Zero-Shot Coordination remains a central challenge in multi-agent reinforcement learning (MARL). It gets even more challenging when the partners are humans. 
Previous works have some fundamental limitations that prevent them from being deployed in realistic human-AI teaming settings.

Firstly, owing to the diversity among human partners, agents must adapt their behavior to each individual. A common strategy is to train against diverse partners to capture this diversity~\citep{Bard2020, Tesauro1994, Jaderberg2017, Carroll2019, Lupu2021, Canaan2022, Zhao2021}. While exposure to diverse behaviors addresses the challenge that the true human strategy may be unknown and enables agents to coordinate effectively with previously unseen partners during deployment~\citep{hu2020other, zsc1}, it does not guarantee that the agent learns a truly adaptive policy~\citep{zsc1, zsc2}. Agents may learn conformant policies~\citep{conf} that perform reasonably across partners without reasoning about partner type, or may even ignore the partner entirely~\citep{zsc1, inter, zsc2}. Secondly, most works overlook an important aspect of human-AI teaming: \emph{mutual adaptation}. In realistic settings, humans do not act independently of their teammates—they continuously adapt in response to the agent’s behavior~\citep{mutual1,h-r1}. 
Traditional approaches often focus on the agent adapting to a human, but do not consider how the human adapts in response~\citep{r-h1, tom2c, ma2024fast}. In reality, as the AI adapts to the human’s strategy, the human observes the agent and adjusts their own behavior in return.

While the problem of adapting to a fixed human type can be formulated as a partially observable problem with hidden partner type~\citep{POMDP-1}, mutual adaptation, however, is more complex: \emph{it requires us to revisit the fundamental way agents are trained.} 
One natural approach is to train two mutually adapting agents, so as to expose the agent to adaptive behavior during training. 
When multiple equilibria exist, jointly training adaptive agents exposes them to such behaviors during training  but often leads to convergence to a single equilibrium and brittle conventions~\citep{matignon, fulda, convention1, hu2020other, Bard2020} that fail to generalize~\citep{convention1, hu2020other, Bard2020}. 

\begin{wrapfigure}{r}{0.42\textwidth}
    \vspace{-12pt}
    \centering
    \includegraphics[width=0.4\textwidth]{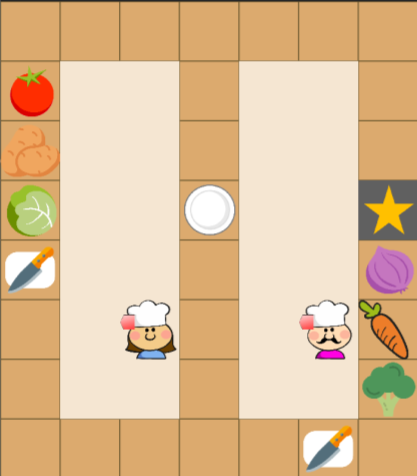}
    \caption{\textbf{Multi-recipe Overcooked domain.} Each agent selects one of three recipes (\textit{PotatoBroccoliSalad}, \textit{LettuceOnionSalad}, or \textit{TomatoCarrotSalad}) per episode. Each recipe requires both agents to contribute one ingredient from their respective sides, place ingredients on a shared plate, and deliver the completed dish to the serving station. Multiple episodes occur within each round, requiring repeated coordination to establish a shared convention on which recipe to prepare.}
    \label{fig:domain}
    \vspace{-8pt}
\end{wrapfigure}
We study these challenges of human-AI teaming in a multi-recipe Overcooked environment (Fig.~\ref{fig:domain}), where a human–AI team must coordinate on one of three recipes without explicit communication, and partners may adapt their strategy over repeated interactions. The agent must reason both about which recipe to pursue and how its actions influence the partner's future behavior. This setting captures two intertwined challenges: multiple equilibria, and mutual adaptation.

We aim to address these challenges by formulating the problem as a bounded level-2 Interactive POMDP~\citep{gmytrasiewicz2005framework}, consistent with findings on bounded human recursive reasoning~\citep{tom1, tom2}. We propose \textbf{NestRL: a nested training regime} that learns this bounded reasoning structure.
In this work, we investigate the following questions: 
\noindent
\refstepcounter{rq}\label{rq1}\textbf{\therq:} Does the nested training regime improve coordination with adaptive partners compared to existing adaptation methods?
\refstepcounter{rq}\label{rq2}\textbf{\therq:} What coordination strategies emerge under nested training when agents interact with adaptive partners?
\refstepcounter{rq}\label{rq3}\textbf{\therq:} Does training with adaptive partners prevent collapse to implicit signaling conventions? 

We evaluate our approach in a modified Overcooked environment~\citep{Carroll2019, charakorn2023} that requires coordination under multiple equilibria, testing generalization to adaptive partners and real humans in a user study. Our method consistently outperforms state-of-the-art baselines for zero-shot cooperation and human-AI teaming, while achieving higher success rates and more stable coordination. Our contributions are:
\begin{itemize}[nosep,leftmargin=*]
\item We formalize mutual adaptation as bounded recursive reasoning grounded in I-POMDPs and introduce \emph{NestRL}\footnote{ For correspondence, contact \href{mailto:ubiswas2@asu.edu}{ubiswas2@asu.edu}. Code is available at
\url{https://github.com/upasana27/adaptive-RL}.}, a nested training regime that learns finitely nested reasoning. 
\item We show that NestRL improves coordination with adaptive partners and prevents collapse to implicit signaling conventions. NestRL agents also learn coordination behaviors such as waiting and initiative-taking, which facilitate mutual adaptation when paired with previously unseen adaptive agents.
\item We validate these findings in a user study with $N{=}25$ human participants, demonstrating that nested training yields agents that achieve higher task success, adapt within rounds, and are perceived as more responsive partners.
\end{itemize}
\textbf{Related Works : }Population-based training exposes agents to diverse stationary partners~\citep{Bard2020, Tesauro1994, Jaderberg2017, Carroll2019, Lupu2021, Canaan2022, Zhao2021}. We additionally study coordinating with partners that adapt back to the agent. Opponent modeling infers partner policies~\citep{he2016opponent, raileanu2018modeling} or differentiates through partner learning~\citep{foerster2017learning}, while co-learning agents often develop brittle conventions~\citep{hu2020other}. Recent latent-inference methods such as LIAM~\citep{papoudakis2021agent}, LILI~\citep{xie2021learning}, and PACE~\citep{ma2024fast} adapt online to partner behavior.
We test whether these methods themselves improve when trained under proposed nested training regime.
I-POMDPs~\citep{gmytrasiewicz2005framework, doshi2005particle, hoang2013interactive, ng2012bayes, han2019ipomdp} model beliefs over other agents' goals and policies but struggle to scale to high-dimensional settings. We approximate finitely nested reasoning via structured training in deep RL.
PR2 and GR2 study recursive reasoning among learning agents and analyze convergence to equilibria~\citep{wen2019pr2,wen2020gr2}, while k-level reasoning has also been used for zero-shot coordination in Hanabi~\citep{cui2021klevel}. We build on these established ideas to train a single agent for interaction with unseen adaptive partners, with whom joint training is not possible. \citet{sriraman2026conventionplay} is the closest to our work and learns probing, leading, and following behaviors under capability limitations. We additionally examine the contribution of nested training across architectures and study how different partner populations produce different adaptation behaviors. Extended discussion in Appendix~\ref{app:related}.



\section{Problem Formulation}\label{sec:markov_game}
\textbf{Cooperative Markov Game.}
We model each episode as an $m$-player, finite-horizon cooperative Markov game $\langle \mathcal{S}, \mathcal{A}, \mathcal{T}, H, \mathcal{R} \rangle$, where $\mathcal{S}$ is the state space, $\mathcal{A} = \mathcal{A}_1 \times \dots \times \mathcal{A}_m$ is the joint action space, $\mathcal{T}$ is the transition function, $H$ is the horizon, and $\mathcal{R}$ defines a shared reward $r(s^h, a^h)$ for cooperative settings. Agent $i$'s Markovian policy is $\pi_i(a_i \mid s)$; the joint policy of all other agents is $\pi_{-i}$. The value of state s under joint policy $(\pi_i, \pi_{-i})$ is $V^i_{(\pi_i, \pi_{-i})}(s) = \mathbb{E}\left[\sum_{h=1}^H r(s^h, \mathbf{a}^h)\right]$. A joint solution $\bm{\pi^*} = (\pi_i^*, \pi_{-i}^*)$ satisfies mutual best response: $\pi_i^* = BR(\pi_{-i}^*)$ and $\pi_{-i}^* = BR(\pi_i^*)$, where $BR(\pi_{-i}) = \argmax_{\pi_i} V^{i}_{(\pi_i,\pi_{-i})}(s^0)$. We assume the game admits multiple solutions $\Pi^* = \{\bm{\pi^1}, \dots, \bm{\pi^J}\}$, each corresponding to a distinct equilibrium in the game. In our setting, this corresponds to committing to a particular recipe.

\textbf{Incompatible Conventions in Human-AI Teaming:}
Using the notation introduced in \cite{OSP}, we call each equilibrium a \emph{convention}. Conventions $\bm{\pi}, \bm{\pi'}$ are \emph{incompatible} if $(\pi_i, \pi'_{-i}) \notin \Pi^*$ or $(\pi'_i, \pi_{-i}) \notin \Pi^*$ i.e. they are not a solution of the Markov Game. All agents must implicitly or explicitly agree on a particular convention; as long as they are not incompatible, coordination succeeds. Existing MARL approaches for solving a Markov Game often converge to a single solution policy $\bm{\pi} \in \Pi^*$. This is problematic in human–robot teaming, where the human may follow a different, incompatible convention $\bm{\pi}^H$. 
This highlights a fundamental limitation of standard MARL: solving for a single equilibrium is not sufficient in human–AI collaboration. In our setting, the agent must reason not only about the convention the partner is following but also about how the behavior of its partner may evolve over repeated interactions.

\textbf{Mutual Adaptation in Human–Robot Teaming.}
This challenge is compounded by mutual adaptation. While prior work typically assumes the robot adapts to a fixed human, humans also adjust their strategy in response to the robot. 
If the robot treats the human as stationary, its actions may unintentionally induce further shifts in the human’s policy. In games with multiple conventions, such bi-directional adaptation need not converge to a joint solution of the Markov Game. Instead, the interaction may oscillate or drift, preventing the team from stabilizing on any convention. Robust coordination therefore requires reasoning not only about current human behavior, but about how that behavior will change in response to the agent's actions. By modeling the human's adaptation, the agent can select actions that steer the interaction toward a stable, shared equilibrium. 
\section{Proposed Framework}
In human–AI cooperation, optimal actions depend not only on the current environment state, but also on evolving beliefs about the partner’s adaptive behavior. 
We model this interaction as a finitely nested I-POMDP, which explicitly represents the other agent within the state space and captures their adaptation strategy. 
To make learning tractable, we introduce \emph{NestRL: a nested training regime} to learn interactive beliefs using latent embeddings.
\subsection{I-POMDP Formulation}
We formalize the human–AI interaction as a finitely nested \emph{Interactive Partially Observable Markov Decision Process} (I-POMDP) as introduced in \cite{gmytrasiewicz2005framework}. 
An I-POMDP extends a standard POMDP by incorporating explicit models of other agents into the state space. For a robot (agent $i$) interacting with a human (say agent $j$), the level-$l$ I-POMDP is defined as:
\[
\mathcal{M}_{i,l} = 
\langle \mathcal{IS}_{i,l}, \mathcal{A}_i, \mathcal{T}_i, \mathcal{R}_i, \Omega_i, \mathcal{O}_i, \gamma \rangle.
\]
where the components are:
\begin{itemize}[nosep]
\item The interactive state is $\mathcal{IS}_{i,l} = S \times M_{j,l-1}$
where $S$ is the environment state and $M_{j,l-1}$ is the set of possible models of agent j at level $l-1$.
\item In our setting, each model $m_j \in M_{j,l-1}$ represents a distinct adaptive human type and is defined as $m_j = \langle b_j, \hat{\pi}_j \rangle$ where $b_j \in \Delta(S \times M_{i,l-2})$ is the human’s belief over the environment and possible Level-(l-2) models, and $\hat{\pi}_j$ is the human’s policy induced by that belief. 
Importantly, in our setting this model encodes: 1) The convention the human follows (e.g., preferred recipe) 2) The human’s adaptation rule in response to the robot.
\item The robot receives observations $\Omega_i$ through $
O_i : S \times A \times \Omega_i \rightarrow [0,1].
$.We assume \emph{Model Non-Observability (MNO)}: the robot cannot directly observe the human’s internal model or beliefs. It must infer them from observed behavior.
\end{itemize}
\textbf{Nested Beliefs : }Finite nesting in an I-POMDP corresponds directly to bounded Theory of Mind (ToM) reasoning. A level-0 agent reasons only about the environment $S$ and does not explicitly model its partner. A level-1 agent maintains beliefs over $S$ and level-0 models of its partner, corresponding to first-order ToM (“I reason about what you will do”). A level-2 agent maintains beliefs over $S$ and level-1 partner models, corresponding to second-order ToM (“I reason about how you reason about me”). More generally, a level-$k$ agent’s belief is a distribution over physical states and level-$(k-1)$ models of the other agent. In this work, we restrict reasoning to level-2 nesting. Concretely, the Level-2 robot reasons about a Level-1 human who adapts to possible Level-0 behaviors. This formulation captures both uncertainty over human type (convention) and uncertainty over how the human adapts in response to the robot using interactive states.

\textbf{Discussion of Modeling Choice : } We adopt a bounded level-2 I-POMDP as the formal framework for mutual adaptation, motivated by two considerations. First, mutual adaptation specifically requires representing that the partner reasons about the agent, not merely that the partner has unobserved properties. Latent-variable methods~\citep{xie2021learning} can in principle capture this implicitly, but they do not incentivize the agent to discover the recursive reasoning structure. The agent must learn from data that the partner is adapting in response to it's own actions. I-POMDPs enable this by explicitly encoding this structure in their formulation, and, therefore motivate our nested training regime. Second, we restrict reasoning to level-2 nesting rather than full recursion. Bounded level-k reasoning is well-established in cognitive science as a model of human strategic behavior~\citep{tom1, tom2}. 

\textbf{Limitations.} The framework assumes partners adapt through recursive reasoning. Partners may adapt through different processes. Humans are occasionally irrational and may not employ recursive reasoning when adapting~\citep{nat}. However, recursive reasoning is a well-established model of human strategic behavior and forms the foundation of the theory of mind literature~\citep{tom1}. We therefore adopt it as a principled modeling choice.
Higher-order partners (level-3 and beyond) are not modeled yet.
\begin{figure}
    \includegraphics[width=1\linewidth]{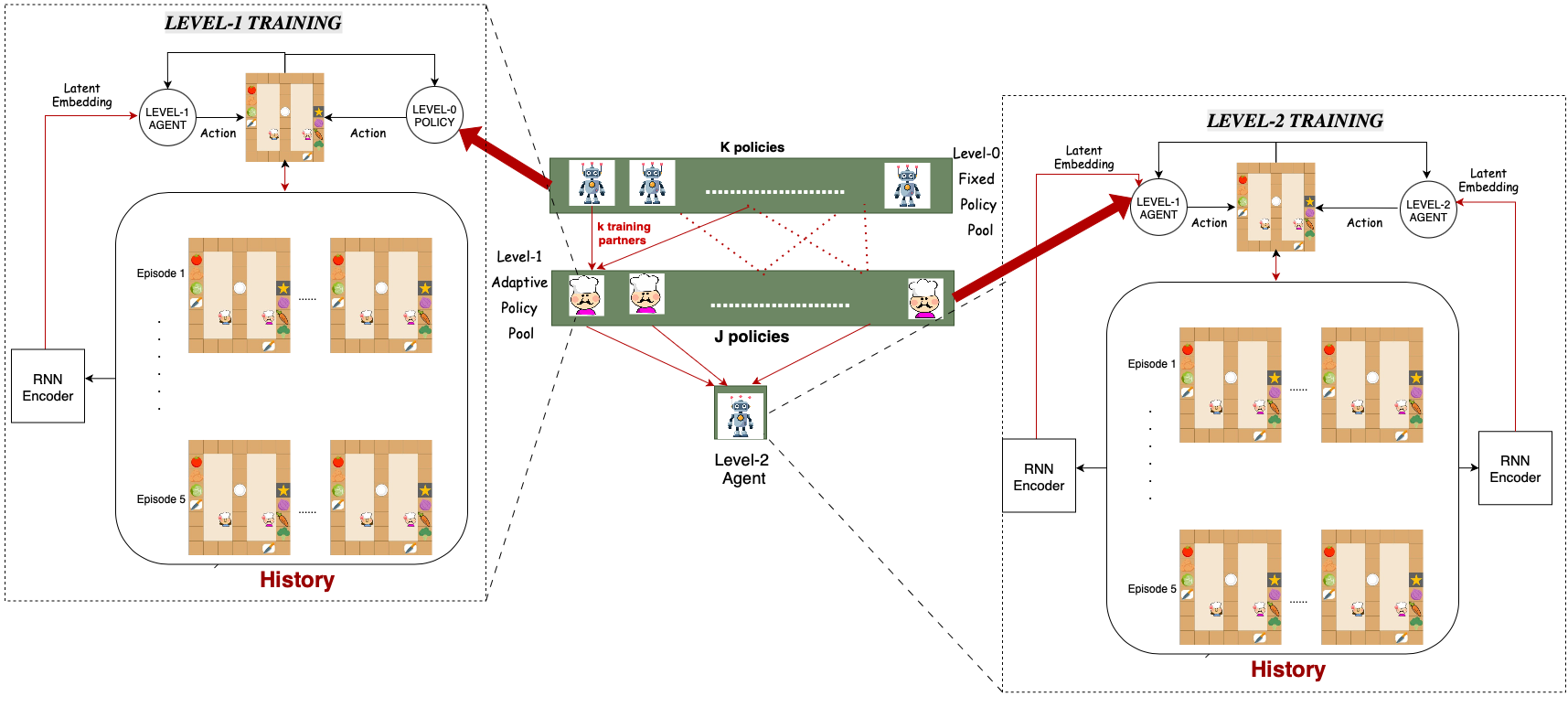}
    \caption{\textbf{NestRL: Nested training for mutual adaptation: }Level-1 human policies are trained against fixed robot policies, producing diverse adaptive behaviors corresponding to different coordination conventions. A level-2 robot is then trained against these adaptive partners, learning to maintain uncertainty over partner types and to act in a way that anticipates how the partner will adapt in return. This models bounded recursive reasoning in an I-POMDP framework and promotes coordination across multiple equilibria.}
    \label{fig:arch}
\end{figure}
\subsection{NestRL: A Nested Training Regime}
Solving a finitely nested I-POMDP exactly is intractable, as it requires maintaining and updating beliefs over recursively defined partner models. 
We therefore introduce \emph{NestRL} to learn level-2 interactive reasoning through a nested training regime, as shown in Fig.~\ref{fig:arch}.
We consider a two-player cooperative Markov game with a robot policy $\pi_R$ and a partner policy $\pi_P$. 
\paragraph{Level-1 Training: }We first train a set of Level-1 partner policies. 
Each Level-1 policy $\pi_P^j$ is trained against a subset of \emph{fixed} Level-0 policies 
$\{\pi_R^1, \dots, \pi_R^K\}$, which act based on the environment state and do not model the partner.
Consequently, the Level-1 agent must adapt to the behavior of these fixed policies, corresponding to \emph{first-order reasoning} where the Level-1 policy reasons about Level-0 behavior.
During this stage, the partner is trained using a POMDP formulation: it observes the environment and the robot’s behavior, and learns to adapt to different level-0 policies. 
This produces a set of adaptive partner policies 
$\{\pi_P^1, \dots, \pi_P^J\}$, which act as \emph{proxies} for adaptive human behaviors at Level-1 reasoning. 

\textbf{Level-2 Training: }We then train the robot policy $\pi_R$ against the set of adaptive partner policies 
$\{\pi_P^1, \dots, \pi_P^J\}$.
From the robot’s perspective, the partner is drawn from this set but is unknown at the start of each episode.
Thus, the robot must act under uncertainty over which Level-1 partner model it is interacting with. The training stage therefore implements a Level-2 I-POMDP approximation, where the interactive state of the robot consists of the environment state and the partner level-1 model. This corresponds to \emph{second-order reasoning} where the robot reasons about a Level-1 partner that itself reasons about Level-0 behavior. In our policy architecture, a recurrent encoder summarizes the interaction history into a learned latent representation, which conditions the policy alongside the current observation. Details of this representation and its training objective are provided in Appendix~\ref{subsec:nonmarkovian}.

Overall, NestRL learns level-2 interactive reasoning by exposing the robot to adaptive partner policies during training, without requiring explicit recursive belief updates. This structure is motivated by the difficulty of learning to cooperate with unrestricted adaptive partners.
\citet{loftin} show that no learning algorithm can guarantee cooperation with all possible adaptive partners in repeated coordination games under unstructured adaptation.
They propose cognitive hierarchies as a possible remedy, where higher-level agents reason about the adaptive behavior of lower-level partners. \emph{NestRL} instantiates this through a nested training regime in which each level learns against fixed lower-level policies. Within this hierarchy, complementary adaptation roles and behaviors emerge. 
Therefore, while \citet{loftin} suggest that cooperation under mutual adaptation is learnable  when adaptation is ordered, this work provides a training regime under which agents learn this ordering, producing behaviors useful for collaboration with adaptive partners, including humans.

\textbf{Maintaining Diversity through Nested Training Regime: }In standard multi-agent training, co-trained agents often converge to a single convention that reflects the specific interaction patterns seen during training~\citep{convention1, hu2020other, Bard2020}, limiting the generalization to partners following different, incompatible conventions.
In contrast, NestRL prevents such collapse by construction, resulting in adaptive policies that remain compatible with multiple joint solutions. 
We distinguish between \textbf{task-specific coordination
conventions}, where each recipe corresponds to a distinct
joint solution, and \textbf{task-irrelevant signaling
conventions}, which are arbitrary behavioral protocols
that emerge during co-adaptation~\citep{hu2020other}.
Throughout this discussion, ``signaling conventions''
refers to the latter.
We now formalize this property.
\begin{definition}[Signaling Conventions]
A joint solution $\pi^*\in\Pi^*$ is a signaling
convention if it relies on an arbitrary, task-irrelevant
behavioral protocol and is incompatible with every other joint solution $\tilde{\pi}^*\in\Pi^*\setminus\{\pi^*\}$.
Let $\Pi_{\mathrm{sig}}^*\subseteq\Pi^*$ denote the set of
joint solutions satisfying this property.
By a slight abuse of notation, we write
$\pi_i\in\Pi_{\mathrm{sig}}^*$ and $\pi_{-i}\in\Pi_{\mathrm{sig}}^*$
when $\pi=(\pi_i,\pi_{-i})\in\Pi_{\mathrm{sig}}^*$.
\label{def:signaling}
\end{definition}
\begingroup
\setlength{\parskip}{0pt}
\setlength{\topsep}{3pt}

\begin{theorem}[Non-emergence of Signaling Conventions under Nested Training]
Consider a Markov game admitting joint solutions $\Pi^*=\{\pi^1,\dots,\pi^m\}$ and signaling conventions $\Pi_{\mathrm{sig}}^*$.
Agents are trained via a nested regime such that a level-$(n+1)$ agent is paired with level-$n$ policies. Suppose:
\begin{assumption}
The level-0 policies are hand-designed and contain no signaling conventions. Thus, $\pi_0^j\notin\Pi_{\mathrm{sig}}^*$ for every $j\in\{1,\dots,k_0\}$.
\label{A:l0}
\end{assumption}
\begin{assumption}
At each level $n\geq1$, the agent learns a policy $\pi_n$ that maximizes the expected return against the set of lower-level policies $\{\pi_{n-1}^1,\dots,\pi_{n-1}^{k_{n-1}}\}$, with the partner drawn uniformly from the set at the start of each episode. Moreover, the set is realizable: there exists a history-conditioned policy $\tilde{\pi}_n$ such that $(\tilde{\pi}_n,\pi_{n-1}^j)\in\Pi^*$ for every $j\in\{1,\dots,k_{n-1}\}$, and $k_{n-1}\geq2$.
\label{A:nested}
\end{assumption}
\begin{assumption}
The lower-level policies remain frozen throughout the level-$n$ agent's training.
\label{A:frozen}
\end{assumption}
Then, at every level $n\geq1$, the learned policies do not jointly develop or learn any signaling conventions and satisfy $\pi_n\notin\Pi_{\mathrm{sig}}^*$. The training regime prevents the emergence and propagation of such signaling strategies throughout the training regime.
\label{theorem:nested-signaling}
\end{theorem}

\begin{proof}
We prove this by the principle of mathematical induction.
\par\smallskip
\noindent\textit{Base Case ($n=1$):}
The level-1 learner trains against the level-0 policies $\{\pi_0^1,\dots,\pi_0^{k_0}\}$.
Under Assumption~\ref{A:l0}, these policies are not associated with any signaling conventions, i.e., $\pi_0^j\notin\Pi_{\mathrm{sig}}^*$. If $\pi_1$ was associated with a signaling convention, its only compatible partner would also belong to $\Pi_{\mathrm{sig}}^*$ by Definition~\ref{def:signaling}.
For each $j$, let $V_j^*$ denote the best-response value against $\pi_0^j$, and let $V_j(\pi)$ denote the value of $(\pi,\pi_0^j)$. For any policy $\pi$, the pooled return satisfies
$V_{\mathrm{pool}}(\pi)=\frac{1}{k_0}\sum_{j=1}^{k_0}V_j(\pi)\leq\frac{1}{k_0}\sum_{j=1}^{k_0}V_j^*$,
with equality if and only if $V_j(\pi)=V_j^*$ for every $j$. Under Assumption~\ref{A:nested}, the realizable policy $\tilde{\pi}_1$ attains equality, so any maximizer of the pooled objective attains equality as well. Hence $\pi_1$ is a best response to each $\pi_0^j$, and $(\pi_1,\pi_0^j)\in\Pi^*$ for every $j\in\{1,\dots,k_0\}$. Since $\pi_0^j\notin\Pi_{\mathrm{sig}}^*$ for every $j\in\{1,\dots,k_0\}$, it follows that $\pi_1\notin\Pi_{\mathrm{sig}}^*$.
Furthermore, under Assumption~\ref{A:frozen}, the level-0 partners remain fixed. Therefore, the learner and its partners cannot jointly establish a new private signaling strategy through co-training. Consequently, the level-1 learner does not learn a signaling convention, i.e. $\pi_1\notin\Pi_{\mathrm{sig}}^*$.
\par\smallskip
\noindent\textit{Induction Hypothesis:}
Assume that no policy constructed through level $n$ belongs to $\Pi_{\mathrm{sig}}^*$, i.e., $\pi_\ell^j\notin\Pi_{\mathrm{sig}}^*$ for every $0\leq\ell\leq n$ and $j\in\{1,\dots,k_\ell\}$.
\par\smallskip
\noindent\textit{Induction Step ($n\rightarrow n+1$):}
The level-$(n+1)$ learner trains against the policies $\{\pi_n^1,\dots,\pi_n^{k_n}\}$.
By the induction hypothesis, $\pi_n^j\notin\Pi_{\mathrm{sig}}^*$ for every $j\in\{1,\dots,k_n\}$. By the same argument as in the base case, Assumption~\ref{A:nested} implies that $\pi_{n+1}$ is a best response to each $\pi_n^j$, and $(\pi_{n+1},\pi_n^j)\in\Pi^*$ for every $j$. If $\pi_{n+1}$ were a signaling strategy, its only compatible partner would also belong to $\Pi_{\mathrm{sig}}^*$ by Definition~\ref{def:signaling}. Since $\pi_n^j\notin\Pi_{\mathrm{sig}}^*$, it follows that $\pi_{n+1}\notin\Pi_{\mathrm{sig}}^*$. Moreover, under Assumption~\ref{A:frozen}, these partners remain fixed throughout training. Consequently, the learner and its partners cannot jointly establish a new signaling convention through co-training. It follows that $\pi_{n+1}\notin\Pi_{\mathrm{sig}}^*$.
\par\smallskip
\noindent
By induction, no signaling strategy associated with $\Pi_{\mathrm{sig}}^*$ is jointly developed or learned at any level of nested training. Therefore, in contrast to a training regime where the learners co-adapt simultaneously, NestRL systematically prevents collapse to a single joint solution by preventing the emergence and propagation of task-irrelevant signaling conventions.
\end{proof}
\endgroup

\setlength{\textfloatsep}{9pt plus 2pt minus 2pt}
\setlength{\floatsep}{7pt plus 2pt minus 2pt}
\setlength{\intextsep}{6pt plus 2pt minus 2pt}
\setlength{\emergencystretch}{1em}
\raggedbottom
\section{Experimental Setup\protect\footnote{We also evaluate an LLM-based ego agent
(Appendix~\ref{app:llm_setup}) on the Overcooked domain. We test the nested training
regime in Level-Based Foraging
(Appendix~\ref{app:lbf}). Setup described in Appendix.}}
In this section, we present the experimental setup. Further details are provided in Appendix~\ref{app:exp}.
\textbf{Domain: }We evaluate NestRL in the Overcooked cooking environment~\citep{Carroll2019}, a standard benchmark for human–AI coordination. We adopt a multi-recipe, required-coordination variant~\citep{charakorn2023} (Fig.~\ref{fig:domain}) in which agents must complete \emph{one of multiple} distinct recipes per episode. We consider two settings where the ego agent is paired with different types of adaptive partners: (1) \textbf{homogeneous capabilities}, where every partner can prepare every recipe, and (2) \textbf{heterogeneous capabilities} (Fig,~\ref{fig:envs}), where each partner can prepare only a subset of recipes, with the subset varying across partners. These settings require NestRL to learn different adaptive behaviors. While both require the ego agent to learn to coordinate with an adaptive partner, in the latter, it must additionally infer the partner's capabilities i.e. the recipes it can prepare from the partner's behavior.

\textbf{Baselines: }
We compare against state-of-the-art adaptive MARL baselines that represent the dominant paradigms for training agents who can adapt to a diverse set of partners: 1) \textit{LIAM}~\citep{papoudakis2021agent}: learns a partner model via auxiliary prediction of partner observations and actions. 2) \textit{LILI}~\citep{xie2021learning}: encodes partner behavior implicitly through cross-episode context. 3)  \textit{Generalist}: leverages cross-episode adaptation without explicit opponent modeling. 4)  \textit{PACE}~\citep{ma2024fast}: performs context-aware exploration and incorporates a peer-identification auxiliary objective.
All baselines are trained under their original training protocols.

\textbf{Evaluation Protocol: }
Following prior work evaluation of multi-agent RL methods with independently trained partners~\citep{hu2020other,Lupu2021}, we construct separate training and evaluation populations of adaptive Level-1 partners using disjoint random seeds. We evaluate all agents with held-out adaptive Level-1 partners. For this, we generate adaptive partners that were not seen by the Level-2 agent during training by using random seeds. At test time, agents trained are paired with these unseen Level-1 partners.
For each architecture, agents trained under the original and nested regimes are evaluated with the same held-out Level-1 partners of that architecture. 
We evaluate each agent in both player positions over $10$ rounds of $5$ episodes per partner and position. 
\section{Results}\label{sec:results}
\subsection{Performance with Adaptive Partners\protect\footnote{
Additional results for the LLM-based ego agent in Overcooked
and nested training for Level-Based Foraging are reported in Appendices~\ref{app:llm_results} and~\ref{app:lbf_results}, respectively.}}
To test \ref{rq1}, we examine whether nested training improves coordination across different architectures. We compare baselines under their original training protocols (against L0 partners only) as well as under the proposed nested training regime. This allows us to isolate the contribution of the nested training regime from that of the architecture used to instantiate it, and also tests whether the benefits of the training regime extends across different architectures. We evaluate the resulting agents with held-out adaptive L1 partners.

We report the performance of the agents in both settings, where the ego agent is paired with adaptive partners with homogeneous vs heterogeneous abilities. Since the ego agent has to adapt and coordinate differently with adaptive partners in each setting, we test whether NestRL is able to learn the behaviors required for coordination in both settings.
In the first setting, as shown in Table~\ref{tab:adaptive-eval}, NestRL achieves an average success rate of $82\%$ with adaptive partners, substantially outperforming the baselines. We also measure how much an agent's performance changes across different partners using the coefficient of variation (CV) as defined in Table~\ref{tab:adaptive-eval}. NestRL achieves a high average success rate with a low CV of $0.03$, indicating that it is able to learn a policy that adapts and performs consistently well across all evaluation partners. On the other hand, though the baselines achieve high performance with non-adaptive L0 partners ($95\%$), their performance degrades significantly when paired with adaptive partners, thus indicating that baselines learn a policy that works with only static partners and struggle to generalize to partners that adapt back to the agent.
In the second setting, when paired with adaptive partners that have heterogeneous capabilities differing in which recipes they can prepare, we report that NestRL achieves an average success rate of $77\%$ with a CV of $0.12$, which is significantly higher compared to the baselines. 
This shows that NestRL can also successfully coordinate with adaptive partners with different capabilities, where successful coordination additionally requires inferring the capabilities of the partner and adjusting its behavior accordingly. Therefore, using the proposed nested training regime, NestRL is successfully able to learn different adaptation behaviors needed to coordinate with partners in both settings.
\begin{table*}[!tb]
\centering
\small
\renewcommand{\arraystretch}{1.10}
\setlength{\tabcolsep}{4.5pt}
\resizebox{\linewidth}{!}{%
\begin{tabular}{l cccc cc}
\toprule
& \multicolumn{4}{c}{\textbf{Homogeneous capabilities}}
& \multicolumn{2}{c}{\textbf{Heterogeneous capabilities}} \\
\cmidrule(lr){2-5} \cmidrule(lr){6-7}
& \multicolumn{2}{c}{Original Training Regime}
& \multicolumn{2}{c}{Nested Training Regime}
& \multicolumn{2}{c}{} \\
\cmidrule(lr){2-3} \cmidrule(lr){4-5}
\textbf{Architecture}
& Avg. Success & CV & Avg. Success & CV & Avg. Success & CV \\
\midrule
\textsc{PACE}
& 0.56 (0.53--0.59) & 0.16
& \textbf{0.83} (0.80--0.86) & 0.08
& 0.38 (0.35--0.41) & 0.25 \\
\textsc{Generalist}
& 0.53 (0.51--0.55) & 0.09
& \textbf{0.87} (0.84--0.90) & 0.17
& 0.32 (0.29--0.35) & 0.30 \\
\textsc{LiAM}
& 0.50 (0.48--0.52) & 0.17
& \textbf{0.70} (0.66--0.74) & 0.20
& 0.41 (0.38--0.44) & 0.29 \\
\textsc{LiAMx}
& 0.35 (0.33--0.37) & 0.31
& \textbf{0.92} (0.90--0.94) & 0.03
& 0.28 (0.25--0.31) & 0.38 \\
\textsc{LiLi}
& 0.44 (0.43--0.46) & 0.13
& \textbf{0.68} (0.64--0.72) & 0.27
& 0.36 (0.33--0.39) & 0.31 \\
\textbf{NestRL (Ours)}
& 0.63 (0.61--0.64) & 0.13
& \textbf{0.82} (0.78--0.85) & 0.03
& 0.77 (0.72--0.81) & 0.12 \\
\bottomrule
\end{tabular}
}
\caption{Success rates with held-out adaptive partners (10 rounds $\times$ 5 episodes). Under homogeneous capabilities, we compare original and nested training for each architecture. Under heterogeneous capabilities, baselines use their original training protocols, while NestRL uses nested training. Parenthesized ranges report 95\% Wilson score confidence intervals. CV $= \sigma/\mu$ measures variability across partners, where $\sigma$ and $\mu$ are the standard deviation and mean of partner-specific success rates.
Under homogeneous capabilities, paired tests across held-out partners show significant improvements with nested training for LiAMx ($t_9=12.4$, $p<10^{-3}$), Generalist ($t_9=5.56$, $p<10^{-3}$), NestRL ($t_9=3.79$, $p=0.004$), and LiAM ($t_9=3.01$, $p=0.015$).
Pooled tests over episodes show significant improvements for these all architectures ($\chi^2 \geq 68$, $p<10^{-3}$).}
\label{tab:adaptive-eval}
\end{table*}
To isolate the contribution of the nested training regime from the architecture used to instantiate it, we also train the baseline methods under the proposed nested regime and report performance with the held-out adaptive partners. We report that all baselines evaluated achieve higher success rates with adaptive partners under nested training than under their original training regimes. This shows that the improvement in coordination with adaptive partners is not limited to the architecture we tested and is also observed when the baseline architectures themselves are used to instantiate the nested training regime, thus strongly validating the motivation behind nested training.

\subsection{Qualitative Analysis: Co-Adaptation Dynamics\protect\footnote{We refer to the NestRL agent as L2 in this section, since it learns level-2 reasoning.}}\label{sec:qualitative}
To study \ref{rq2}, we ground our analysis in the behavioral taxonomy of \cite{vanzoelen2021coadaptation}, which characterizes successful human-robot teamwork through patterns such as waiting and initiative-taking. We discuss the observed behaviors below and illustrate them in Figure~\ref{fig:qualitative}. Further details and statistical tests are provided in Appendix~\ref{app:qual}.
\begin{figure}[!tb]
    \centering
    \includegraphics[width=\linewidth]{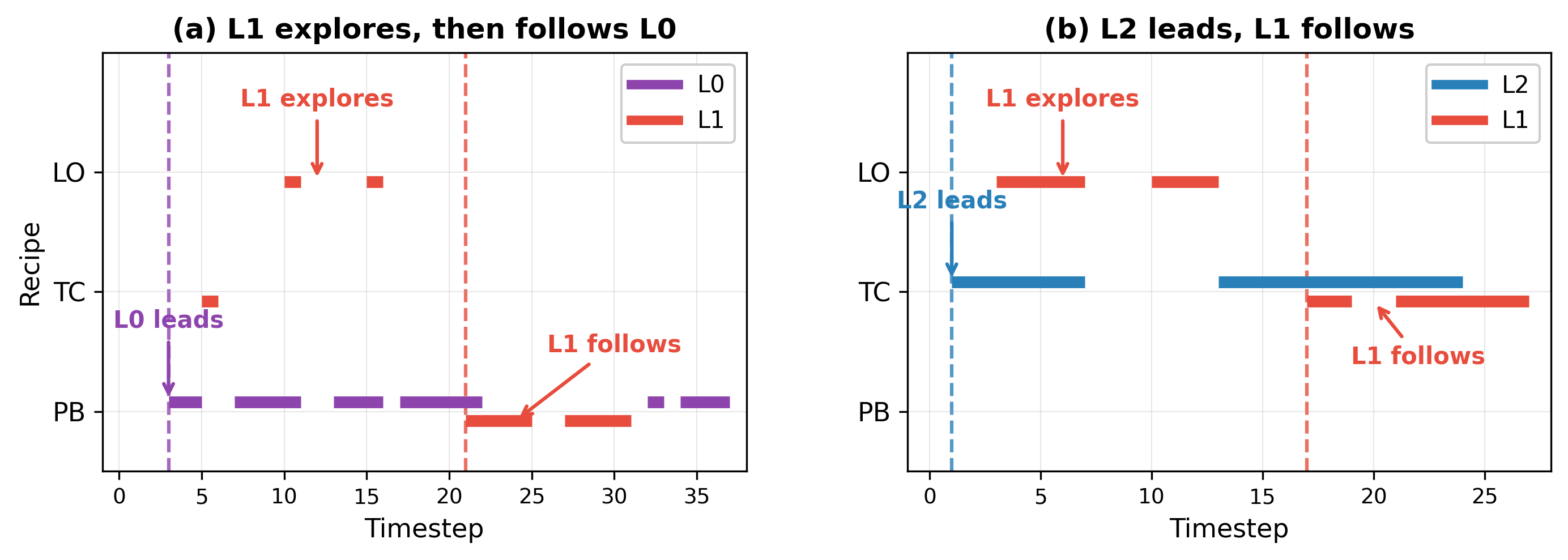}
    \caption{\textbf{Emergent coordination patterns.} Recipe commitment trajectories for representative episodes. Horizontal bars show which recipe each agent works on over time; vertical dashed lines mark key coordination moments. \textbf{(a)}~L1 explores multiple recipes while waiting for L0 to reveal its commitment, then follows L0's lead on the Potato-Broccoli recipe. \textbf{(b)}~L2 takes initiative by committing to Tomato--Carrot at $t{=}1$; L1 initially explores Lettuce--Onion, then follows L2's lead at $t{=}17$.}
    \label{fig:qualitative}
\end{figure}

\textbf{L1 learns to wait under uncertainty.}
When paired with a fixed L0 partner whose convention is unknown, L1 delays irreversible actions until L0 reveals its recipe type, committing after L0 in $85.5\%$ of episodes in the homogeneous capabilities setting (median lead $5$ steps; $W = 12{,}062$, $p < .001$, $d = 0.97$) and in $94.0\%$ of episodes in heterogeneous capabilities setting (median lead $16$ steps; $W = 1{,}925$, $p < .001$, $d = 1.55$), matching the \emph{waiting for team member to act} pattern in human-robot teams. 

\textbf{L2 learns to take initiative and probe.}
Trained against adaptive L1 partners that wait, L2 learns to break the symmetry by committing first.
It commits before its partner in $79.6\%$ of episodes in the homogeneous capabilities setting ($W = 8{,}954$, $p < .001$, $d = -0.97$) and in $88.1\%$ of episodes in the heterogeneous capabilities setting ($W = 2{,}242$, $p < .001$, $d = -1.46$). In the former setting, where every L1 can complete every recipe, L2 \emph{leads}: it picks up its first ingredient at step~4 (median), commits to a single recipe in $89\%$ of episodes, and L1 adopts L2's recipe in $84\%$ of the episodes L2 leads. 
 In the latter, where the partner's capability is unknown, L2 \emph{probes} before committing: it picks up a single ingredient and holds it as a query. L1 \emph{reacts} to the probe: it picks up the complementary ingredient faster when L2's query is in its type than when it is not (median $12$ vs.\ $29$ steps; $U = 3{,}335$, $p < .001$). L2 acts on the reaction: when its query is not compatible with the partner, it switches to the partner's type in $61\%$ of episodes, and only switched episodes succeed ($21\%$ vs.\ $0\%$, $p = .018$). 
L1 follows in both domains, confirming a leader--follower organization matching the \emph{initiative-taking} and \emph{following} co-adaptation patterns. On the other hand, when baseline agents are paired with held-out L1 partners in setting with homogeneous capabilities, baselines commit first in only $23.3$--$46.0\%$ of episodes (vs.\ $79.6\%$ for L2; Fisher's exact test against each baseline, all $p < .001$), leading to deadlocks and miscoordination.

This analysis shows that agents trained using the nested training regime learn behaviors consistent with co-adaptation patterns documented in human-robot teams~\citep{vanzoelen2021coadaptation}. We also observe that the same training regime enables NestRL to learn different adaptation behaviors depending on the partners it is paired with. L1 learns to wait when its partner's type is unknown, and L2, trained against such waiting partners, learns to take initiative. When partners differ in their capabilities, L2 also learns to probe its partner and use the response to adjust its behavior. These behaviors emerge through nested training, without being specified in advance.
\subsection{Empirical Analysis of Signaling Strategies across Training Regimes}
\label{subsec:collapse}
To test \ref{rq3}, we compare NestRL which was trained using the nested training regime against IPPO~\citep{ippo}, a co-adaptation regime in which agents update their policies simultaneously. We examine whether agents trained under these regimes learn signaling conventions.

As defined in Def.~\ref{def:signaling}, agents using such strategies
are incompatible with partners that do not share their signaling conventions, since such conventions may develop as the ego agent and its training partner adapt to each other during training. We therefore compare each agent's performance with its training partners (seen) and held-out partners (unseen) to examine whether the coordination it learned depends on signaling strategies learned with particular training partners. 
\begin{figure*}[!tb]
\centering
\begingroup
\definecolor{XPdark}{RGB}{8,48,107}
\definecolor{XPlight}{RGB}{247,251,255}
\newcommand{\XPcell}[2][]{%
  \cellcolor{XPdark!#2!XPlight}%
  \ifnum#2>59\relax\color{white}\fi
  #1#2%
}
\newcommand{\XPdrop}[1]{%
  #1\,{\fontsize{6.2}{7}\selectfont\textcolor{black!60}{[TBD]}}%
}
\begin{minipage}[t]{0.31\textwidth}
\vspace{0pt}
\centering
{\fontsize{8}{9}\selectfont\textbf{(a) Co-trained Agents}\par}
\vspace{3pt}
{\fontsize{7}{8.2}\selectfont
\setlength{\tabcolsep}{1.0pt}
\setlength{\arrayrulewidth}{0.30pt}
\arrayrulecolor{white}
\renewcommand{\arraystretch}{1.02}
\resizebox{\linewidth}{!}{%
\begin{tabular}{@{}r@{\hspace{1.5pt}}|r|r|r|r|r|r|r|r|r@{}}
  & 1 & 2 & 3 & 4 & 5 & 6 & 7 & 8 & 9 \\
\hline
1 & \XPcell[\bfseries]{100} & \XPcell{0} & \XPcell{0} & \XPcell{0} & \XPcell{10} & \XPcell{20} & \XPcell{10} & \XPcell{80} & \XPcell{0} \\\hline
2 & \XPcell{0} & \XPcell[\bfseries]{100} & \XPcell{0} & \XPcell{0} & \XPcell{0} & \XPcell{0} & \XPcell{0} & \XPcell{0} & \XPcell{0} \\\hline
3 & \XPcell{0} & \XPcell{0} & \XPcell[\bfseries]{100} & \XPcell{0} & \XPcell{0} & \XPcell{0} & \XPcell{0} & \XPcell{0} & \XPcell{40} \\\hline
4 & \XPcell{0} & \XPcell{0} & \XPcell{0} & \XPcell[\bfseries]{100} & \XPcell{50} & \XPcell{100} & \XPcell{90} & \XPcell{0} & \XPcell{0} \\\hline
5 & \XPcell{100} & \XPcell{0} & \XPcell{0} & \XPcell{100} & \XPcell[\bfseries]{100} & \XPcell{100} & \XPcell{50} & \XPcell{10} & \XPcell{0} \\\hline
6 & \XPcell{0} & \XPcell{0} & \XPcell{0} & \XPcell{80} & \XPcell{50} & \XPcell[\bfseries]{100} & \XPcell{100} & \XPcell{0} & \XPcell{0} \\\hline
7 & \XPcell{0} & \XPcell{0} & \XPcell{0} & \XPcell{90} & \XPcell{50} & \XPcell{100} & \XPcell[\bfseries]{100} & \XPcell{0} & \XPcell{0} \\\hline
8 & \XPcell{100} & \XPcell{0} & \XPcell{0} & \XPcell{0} & \XPcell{0} & \XPcell{0} & \XPcell{0} & \XPcell[\bfseries]{100} & \XPcell{0} \\\hline
9 & \XPcell{0} & \XPcell{0} & \XPcell{20} & \XPcell{0} & \XPcell{0} & \XPcell{0} & \XPcell{0} & \XPcell{0} & \XPcell[\bfseries]{100} \\\hline
\end{tabular}
}\par}
\vspace{3pt}
{\fontsize{7}{8}\selectfont
Rows: ego; columns: partner.\par}
\end{minipage}\hfill
\begin{minipage}[t]{0.21\textwidth}
\vspace{0pt}
\centering
{\fontsize{8}{9}\selectfont\textbf{(b) Within-recipe}\par}
\vspace{3pt}
{\fontsize{7}{8.2}\selectfont
\setlength{\tabcolsep}{1.0pt}
\setlength{\arrayrulewidth}{0.30pt}
\arrayrulecolor{white}
\renewcommand{\arraystretch}{1.02}
$\mathcal{R}_1$ (LO)\par
\vspace{2pt}
\begin{tabular}{@{}r@{\hspace{1.5pt}}|r|r@{}}
  & 3 & 9 \\
\hline
3 & \XPcell[\bfseries]{100} & \XPcell{40} \\\hline
9 & \XPcell{20} & \XPcell[\bfseries]{100} \\\hline
\end{tabular}\par
\vspace{4pt}
$\mathcal{R}_2$ (TC)\par
\vspace{2pt}
\resizebox{\linewidth}{!}{%
\begin{tabular}{@{}r@{\hspace{1.5pt}}|r|r|r|r|r|r@{}}
  & 1 & 4 & 5 & 6 & 7 & 8 \\
\hline
1 & \XPcell[\bfseries]{100} & \XPcell{0} & \XPcell{10} & \XPcell{20} & \XPcell{10} & \XPcell{80} \\\hline
4 & \XPcell{0} & \XPcell[\bfseries]{100} & \XPcell{50} & \XPcell{100} & \XPcell{90} & \XPcell{0} \\\hline
5 & \XPcell{100} & \XPcell{100} & \XPcell[\bfseries]{100} & \XPcell{100} & \XPcell{50} & \XPcell{10} \\\hline
6 & \XPcell{0} & \XPcell{80} & \XPcell{50} & \XPcell[\bfseries]{100} & \XPcell{100} & \XPcell{0} \\\hline
7 & \XPcell{0} & \XPcell{90} & \XPcell{50} & \XPcell{100} & \XPcell[\bfseries]{100} & \XPcell{0} \\\hline
8 & \XPcell{100} & \XPcell{0} & \XPcell{0} & \XPcell{0} & \XPcell{0} & \XPcell[\bfseries]{100} \\\hline
\end{tabular}
}\par}
\end{minipage}\hfill
\begin{minipage}[t]{0.45\textwidth}
\vspace{0pt}
\centering
{\fontsize{8}{9}\selectfont
\textbf{(c) Seen versus unseen partners}\par}
\vspace{3pt}
{\fontsize{8}{9.2}\selectfont
\arrayrulecolor{black}
\setlength{\tabcolsep}{3pt}
\setlength{\aboverulesep}{1.6pt}
\setlength{\belowrulesep}{1.8pt}
\renewcommand{\arraystretch}{1.04}
\resizebox{\linewidth}{!}{%
\begin{tabular}{@{}llrrr@{}}
\toprule
\textbf{Method} & \textbf{Regime}
& \shortstack{\textbf{Seen}\\Partners}
& \shortstack{\textbf{Unseen}\\Partners}
& \shortstack{\textbf{Drop}\\(pp)} \\
\midrule
\rowcolor{gray!8}
IPPO & Co-trained & 100.0 & 18.75 & 81.25 \\
\midrule
PACE & Nested & 98.0 & 83.0 & 15.0 \\
LiAM & Nested & 97.7 & 70.0 & 27.7 \\
\rowcolor{gray!5}
LiAMx & Nested & 96.9 & 92.0 & 4.9 \\
LiLi & Nested & 92.9 & 68.0 & 24.9 \\
\rowcolor{gray!5}
Generalist & Nested & 97.1 & 87.0 & 10.1 \\
\midrule
\rowcolor{XPdark!7}
\textbf{NestRL (Ours)} & Nested
& \textbf{95.5} & \textbf{82.0} & \textbf{13.5} \\
\bottomrule
\end{tabular}
}\par}
\end{minipage}
\par
\endgroup
\vspace{2pt}
\caption{\textbf{Coordination with seen and unseen partners.}
\textbf{(a)} Cross-play matrix for nine co-trained IPPO agents from different seeds; darker cells indicate higher success. Diagonal entries (bold) show coordination with the original training partner, while off-diagonal entries show coordination with partners from other runs. \textbf{(b)} Within-recipe cross-play: even agents that converged to the same recipe can fail to coordinate, indicating incompatibility beyond recipe choice. \textbf{(c)} Success rates ($\%$) with training (seen) and held-out (unseen) partners. Drop is seen minus unseen success rates. Co-trained agents exhibit a much larger drop than the six nested-trained agents with reported results, suggesting greater dependence on conventions shared with particular training partners.}
\label{fig:crossplay}
\end{figure*}
As shown in Figure~\ref{fig:crossplay}, co-trained agents achieve $100\%$ success with their training partners, but only $18.75\%$ with unseen partners, a drop of $81.25$ percentage points. In contrast, the six nested-trained methods with reported results achieve high success with both seen and unseen partners, with drops of only $4.9$--$27.7$ percentage points, significantly smaller than the drop for co-trained agents ($z$-tests comparing differences in drops, all $p<10^{-10}$; Appendix~\ref{app:convention-stats}). 
We further evaluate cross-play among agents that converged to the same recipe (Figure~\ref{fig:crossplay}b). Even within the same recipe, agents frequently fail to coordinate, thus suggesting that task-irrelevant signaling conventions developed during training.

Therefore, agents trained using a co-training regime strongly depend on signaling conventions developed with their training partners.
In contrast, agents trained using nested training regime can coordinate with partners not seen during training and learn coordination that is not restricted to signaling conventions shared only with their training partners.
These results provide empirical evidence consistent with Theorem~\ref{theorem:nested-signaling} and show that nested training supports coordination with unseen partners across diverse conventions. We next evaluate whether NestRL’s ability to coordinate extends to actual human partners through a user study.

\subsection{User Study}
We compare NestRL against the strongest baseline, PACE, in a within-subject user study with $N{=}25$ participants. Each participant plays two rounds of 5 episodes with each agent. After each round, participants complete the NASA Task Load Index~\citep{nasa-tlx} and rate partner proactiveness and adaptiveness. We test three hypotheses: \textbf{H1: Task performance.} NestRL should achieve higher success rate than PACE on unseen adaptive human partners. \textbf{H2: Within-round adaptation.} NestRL should improve delivery rate across consecutive episodes within a round as it observes human behavior; this effect should be weaker or absent for PACE. \textbf{H3: Perceived workload.} Participants should report lower workload and perceive NestRL as more responsive.
\begin{wrapfigure}{r}{0.40\linewidth}
\centering
\begin{tikzpicture}
\begin{axis}[
    width=\linewidth,
    height=3.4cm,
    label style={font=\scriptsize},
    xlabel={Episode within round},
    ylabel={Delivery rate (\%)},
    xtick={1,2,3,4,5},
    ymin=40, ymax=100,
    legend style={at={(0.02,0.98)}, anchor=north west, font=\tiny, draw=none, inner sep=1pt},
    grid=major,
    grid style={dashed, gray!30},
    tick label style={font=\scriptsize},
]
\addplot[blue, mark=square*, thick] coordinates {(1,46)(2,54)(3,53)(4,52)(5,50)};
\addplot[red, mark=triangle*, thick] coordinates {(1,59)(2,76)(3,80)(4,91)(5,89)};
\legend{PACE, NestRL (ours)}
\end{axis}
\end{tikzpicture}
\setlength{\abovecaptionskip}{3pt}
\caption{Delivery rate improves across episodes within a round as the robot adapts to the partner.}
\label{user-study2}
\end{wrapfigure}
\textbf{Results:}
NestRL achieves a higher episode delivery rate
($79.6\%$ vs.\ $52.0\%$ for PACE; paired $t(24)=2.41$,
$p=.024$, $d_z=0.48$).
Figure~\ref{user-study2} shows delivery rate by episode
position within a round: PACE remains flat while NestRL
trends upward, and a paired $t$-test on slopes confirms
the difference ($t(24)=2.87$, $p=.008$, $d_z=0.57$).
Participants report lower temporal demand with NestRL
($M=4.92$ vs.\ $6.64$; $t(24)=-3.84$, $p=.001$,
$d_z=-0.77$).
It is also rated significantly more adaptive
($M=4.42$ vs.\ $2.56$; $t(24)=2.53$, $p=.018$,
$d_z=0.51$) and somewhat more proactive
($M=4.24$ vs.\ $2.92$; no significant difference).
The results provide converging evidence across all three
hypotheses. Full results are in
Appendix~\ref{app:user-results}.
\section{Conclusion}
We introduced \emph{NestRL}, a nested training regime for mutual adaptation in human-AI collaboration. Grounded in the I-POMDP formalism, NestRL trains agents against frozen policies from the preceding level, allowing them to learn how their behavior influences their partner's adaptive responses without jointly developing private signaling protocols. 
Through experiments in a multi-recipe variant of Overcooked and a user study, we demonstrated that nested training improves coordination with adaptive agents and human partners. We also find that the nested training regime improved coordination with
adaptive partners across different architectures. We also find that NestRL agents learn coordination behaviors previously observed in human-robot teams, and learn different adaptation strategies depending on the capabilities of their partners. By aligning the learning process with the reasoning structure of finitely nested I-POMDPs, our framework provides a principled foundation for developing agents that collaborate effectively with diverse, adaptive partners. 
Future work includes extending NestRL beyond fully cooperative settings to mixed-motive interactions, where the agents' objectives only partially align. We also plan to broaden the range of Level-1 adaptation strategies using human-grounded partner models. Finally, we plan to study coordination under asymmetric information and hidden goals, where the human has private information about the task or preferred outcome. 

\section*{Acknowledgements}
This research is supported in part by ONR (N00014-25-1-2301 and N00014-23-1-2409), DARPA (HR00112520016), DoD RAI (via CMU subcontract 25-00306-SUB-000), and a generous gift from Qualcomm. The research is supported in part by NSF grant 2303019, ONR grant number N00014-26-1-2041 and a JP Morgan Faculty Award.This material is based in part upon work supported by Other Transaction award 1AY2AX000062 from the U.S. Advanced Research Projects Agency for Health (ARPA-H) Platform Accelerating Rural Access to Distributed Integrated Medical Care (PARADIGM) program. The views and conclusions contained in this document are those of the authors and should not be interpreted as representing the official policies, either expressed or implied, of the U.S. Government.

\bibliographystyle{unsrtnat}
\bibliography{references}
\newpage
\appendix
\section{Related Work}
\label{app:related}

\textbf{Population-Based and Opponent Modeling Approaches.}
Population-based training exposes agents to diverse partner behaviors~\citep{Bard2020, Tesauro1994, Jaderberg2017, Carroll2019, Lupu2021, Canaan2022, Zhao2021}.
Fictitious Co-Play trains against self-play policies and their intermediate checkpoints~\citep{FCP}, while TrajeDi, MEP, and LIPO encourage diversity in the partner population(~\citep{Lupu2021,Zhao2021,charakorn2023}). ~\citet{COLE} studies open-ended population learning, and best-response diversity generates partners that require different responses from the learner~\citep{rahman2023brdiv}.
While effective for robustness, diversity alone does not ensure that the learner encounters partners that adapt in response to its behavior.
Opponent modeling approaches infer other agents' policies~\citep{he2016opponent, raileanu2018modeling} or differentiate through learning updates~\citep{foerster2017learning}.
When agents co-learn, coordination may emerge implicitly yet remain brittle or non-task-specific~\citep{hu2020other}.
Other works model uncertainty and adaptation in ad hoc and human--robot teams~\citep{mutual2,mutual1,mutual3}, including the human's willingness to adapt and the selection of complementary policies based on inferred human behavior.
LIAM~\citep{papoudakis2021agent}, LILI~\citep{xie2021learning}, and PACE~\citep{ma2024fast} adapt online to partner behavior via latent inference or meta-learning. Our focus is on proposing a novel training regime. The same learning method is used at each level, and the levels differ in the partners they are trained against. We therefore instantiate the nested regime with different architectures and hold the learning method fixed within each comparison, to examine whether the improvement comes from the regime rather than a particular architecture. ConventionPlay~\citep{sriraman2026conventionplay} also trains against capability-limited adaptive followers and learns probing, leading, and following behaviors. This is closely related to our heterogeneous-capability setting. Its final agent is trained against a mixture of Level-0 and Level-1 policies, whereas our reported Level-2 training uses the Level-1 pool. Our evaluation examines both capability settings, the effect of the regime across architectures, and coordination with human partners.\\

\textbf{I-POMDPs and Type-Based Reasoning.}
I-POMDPs~\citep{gmytrasiewicz2005framework, doshi2005particle} model beliefs over other agents' goals, policies, and beliefs. Tractable variants include I-POMDP Lite~\citep{hoang2013interactive}, Interactive-PF~\citep{doshi2005particle}, and Bayes-Adaptive I-POMDPs~\citep{ng2012bayes}. Type-based reasoning also appears in game-theoretic planning~\citep{albrecht2019reasoning}, and IPOMDP-Net~\citep{han2019ipomdp} integrates differentiable planning with deep learning. However, most I-POMDP methods struggle to scale to high-dimensional environments.
Recursive reasoning and cognitive hierarchies are also established approaches to bounded strategic reasoning~\citep{camerer2004cognitive}. PR2 and GR2 study recursive reasoning among learning agents and analyze convergence to equilibria~\citep{wen2019pr2,wen2020gr2}. K-level reasoning has also been used for zero-shot coordination in Hanabi~\citep{cui2021klevel}. In our setting, the deployment partner is an unseen adaptive agent or human, encountered only at execution time, so joint training with that partner is not possible. We instead train against a population of lower-level policies whose parameters remain frozen. Our approach approximates finitely nested reasoning through structured training in deep RL, avoiding explicit belief updates while preserving recursive reasoning structure.

\textbf{Evaluating Mutual Adaptation.}
MARL benchmarks such as particle environments~\citep{lowe2017multi} and StarCraft~\citep{ellis2023smacv2} provide environments for studying multi-agent coordination. However, such benchmarks often permit independent or greedy strategies. Standard Overcooked~\citep{Carroll2019} provides limited equilibrium diversity. Our evaluation focuses on tasks with multiple incompatible coordination solutions and partners that adapt to the ego agent.
We evaluate in a modified Overcooked variant~\citep{charakorn2023} with distinct coordination equilibria and required coordination. We consider both homogeneous and heterogeneous partner capabilities, which require different adaptation behaviors. Following the held-out-partner evaluation used in ad hoc teamwork and human-AI coordination~\citep{hu2020other,FCP,Lupu2021,Zhao2021}, we test with partners not encountered during training. This evaluates generalization within the generated partner class; the user study additionally evaluates coordination with human partners outside this class. 

\textbf{LLM-Based Coordination.}
LLM-based approaches use pre-trained knowledge and reasoning to coordinate with unfamiliar teammates in Overcooked. ProAgent~\citep{zhang2024proagent} reasons about the partner's intentions and revises these beliefs based on observed behavior, allowing it to adjust its plans without task-specific fine-tuning. HLA~\citep{liu2024hla} combines intention reasoning and language communication with fast action generation for real-time human--AI coordination, while CausalPlan~\citep{causalplan2026} uses learned causal dependencies to improve cooperative planning. Our work studies how coordination behaviors are learned through training against adaptive partner populations. We compare nested-trained agents against ProAgent as an LLM baseline to examine whether pre-trained reasoning is sufficient for the adaptation required in our settings, and how it compares with learning through interaction with adaptive partners. This includes coordination where the agent must infer the partner's capabilities and adjust its behavior in response.
\section{Implementation Details}
\label{app:implementation}

\subsection{Belief Estimation for Non-Markovian Policies}
\label{subsec:nonmarkovian}
Interactive decision-making between adaptive agents is non-Markovian from the perspective of each agent. 
The optimal action at time $t$ depends not only on the current observation, but on the history of interaction, which determines beliefs over the partner’s internal model. 
In the I-POMDP formulation, this dependence appears as recursive belief updates over interactive states. 
Exact belief tracking is intractable.
We approximate this using a learned latent embedding that amortizes belief inference.

\textbf{Latent Interactive State: }
At training level $n$, the ego agent interacts with a finite set of fixed level-$(n-1)$ policies 
$\{\pi_{n-1}^1, \dots, \pi_{n-1}^k\}$. 
The true partner type is hidden and can be inferred from the interaction history. 
We summarize the interaction history into a learned embedding:
\begin{equation}
z_t^{(n)} = f_\theta^{(n)}(h_t^{(n)}), 
\quad 
h_t^{(n)} = \{(o_\tau^{(n)}, a_\tau^{(n)})\}_{\tau=0}^{t}.
\end{equation}

Here $f_\theta^{(n)}$ is implemented as an RNN encoder. 
The latent variable $z_t^{(n)}$ serves as a compact representation of uncertainty over the partner’s type and adaptation pattern.

\textbf{Policy Conditioning: }
The level-$n$ policy conditions on both the current observation and the latent embedding:
\begin{equation}
a_t^{(n)} \sim 
\pi_\theta^{(n)}(a \mid o_t^{(n)}, z_t^{(n)}).
\end{equation}

This replaces explicit I-POMDP belief updates:
\begin{equation}
b_{i,t+1}(is') \propto 
O_i(o_{t+1} \mid is', a_t^i, a_t^j)
\sum_{is} 
T_i(is, a_t^i, a_t^j, is') \, b_{i,t}(is),
\end{equation}
with amortized inference through $z_t^{(n)}$.

\textbf{Training Objective: }
Encoder and policy parameters are optimized end-to-end using PPO:
\begin{equation}
\max_{\theta} 
\mathbb{E}_{\pi_\theta^{(n)}, \pi_{n-1}}
\left[
\sum_{t=0}^{T} 
\gamma^t R(s_t, a_t^{(i)}, a_t^{(j)})
\right],
\end{equation}
where $\pi_{n-1}$ denotes fixed lower-level policies.

The interaction is non-Markovian in the original state space due to nested belief dependence. 
The latent embedding $z_t^{(n)}$ estimates the interactive beliefs, while rendering the policy approximately Markovian in $(o_t, z_t)$. 
Thus, level-$n$ policies implement amortized belief reasoning consistent with a finitely nested I-POMDP. Each successive policy $\pi_n$ learns a mapping in the latent state space that implicitly captures the non-Markovian structure of interactive reasoning.
\subsection{Baselines}

All methods share the same PPO configuration, the same two-layer $128$-unit ReLU MLPs, and the same $5$-episode interaction window, so the differences lie in how each method represents the partner and in what supervises that representation. 
\begin{itemize}
    \item \emph{PACE} uses exactly this encoder, actor, and critic, and adds the peer-identification mechanism of the original method: a linear head on $z$ is trained with cross-entropy to classify the identity of the training partner (coefficient $1.0$), and the agent receives an intrinsic reward proportional to the classification accuracy (coefficient $0.2$, with a $1$M-step warm-up and decayed over $25$M steps). PACE thus supervises the latent with partner labels while NestRL does not.
    \item \emph{Generalist} has no explicit partner representation. It is a single GRU policy (hidden size $128$, shared actor--critic, truncated backpropagation over chunks of $20$ steps) that receives the observation sequence directly and carries its hidden state across the episodes of the window. It is trained jointly against the partner pool, and any inference about the partner is implicit in the recurrent state.
    \item \emph{LiAM} and \emph{LiAMx} are the Generalist network with the agent's last action appended to the input and two auxiliary heads trained to predict the partner's observation and action from the recurrent state (coefficient $1.0$ each); LiAM resets the hidden state at every episode while LiAMx carries it across the window. 
    \item \emph{LiLi} uses the NestRL encoder but trains the latent with an auxiliary transition-prediction loss (coefficient $1.0$) on the previous episode only, and encodes the agent's own actions rather than the partner's.
\end{itemize}
The baseline methods (LILI, LIAM, LIAMX, Generalist, and PACE) follow a similar overall training structure, with differences primarily in auxiliary tasks and the use of exploration bonuses. All baselines are trained using Proximal Policy Optimization (PPO)~\citep{schulman2017proximal, pytorchrl} in the PO-Overcooked environment.
\begin{table*}[!ht]
\centering
\resizebox{\linewidth}{!}{
\begin{tabular}{lccccccccccc}
\toprule
Algorithm & Learning Rate & PPO Clip $\epsilon$ & Entropy Coef & $\gamma$ & GAE $\lambda$ & Batch Size & \# Epochs & \# Mini-batches & Grad Clip (L2) & Activation & Hidden Dims \\
\midrule
Level-n Training & 1e-3 & 0.2 & 0.03 & 0.99 & 0.95 & 72000 & 4 & 18 & 15.0 & ReLU & [128,128] \\
PACE & 1e-3 & 0.2 & 0.03 & 0.99 & 0.95 & 72000 & 4 & 18 & 15.0 & ReLU & [128,128] \\
Generalist & 1e-3 & 0.2 & 0.03 & 0.99 & 0.95 & 72000 & 4 & 18 & 15.0 & ReLU & [128,128] \\
LILI & 1e-3 & 0.2 & 0.03 & 0.99 & 0.95 & 72000 & 4 & 18 & 15.0 & ReLU & [128,128] \\
LIAM & 1e-3 & 0.2 & 0.03 & 0.99 & 0.95 & 72000 & 4 & 18 & 15.0 & ReLU & [128,128] \\
LIAMX & 1e-3 & 0.2 & 0.03 & 0.99 & 0.95 & 72000 & 4 & 18 & 15.0 & ReLU & [128,128] \\
\bottomrule
\end{tabular}
}
\caption{Hyperparameters for all algorithms in the Overcooked environment.}
\label{tab:hyperparameters}
\end{table*}
\section{Experimental Setup\protect\footnote{Codebase is available at this github repository - https://github.com/upasana27/adaptive-RL.git}}
\label{app:exp}
\begin{figure*}[t]
\centering
\includegraphics[width=\textwidth]{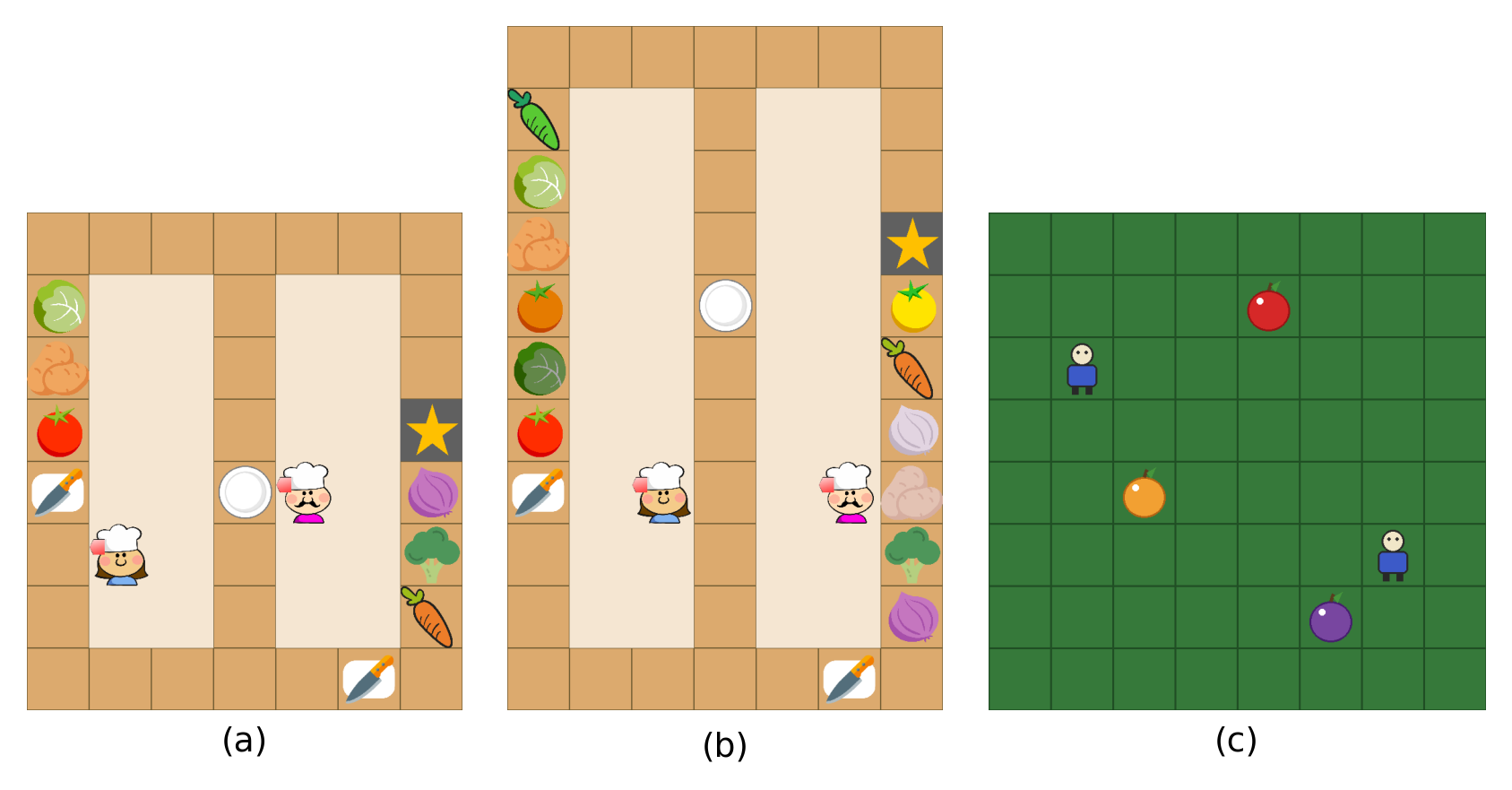}
\caption{Environments. (a)~Overcooked with homogeneous capabilities: a $7\times 8$ kitchen with six ingredients and three recipes (Lettuce+Onion, Potato+Broccoli, Tomato+Carrot); every partner can prepare every recipe. (b)~Overcooked with heterogeneous capabilities: a $7\times 11$ kitchen with twelve ingredients and six recipes in three ingredient-disjoint pairs (Type~1: Lettuce+Onion, Potato+Broccoli; Type~2: Tomato+Carrot, Cucumber+Pepper; Type~3: Spinach+Mushroom, Pumpkin+Garlic); each partner can prepare only the two recipes of its type. In both kitchens the agents cannot cross the divider, each recipe needs one ingredient from each side, both ingredients must be chopped and placed on the shared plate, and only the right agent can reach the delivery square (star). (c)~Level-Based Foraging: an $8\times 8$ grid with one fruit of each of three types; a fruit is collected only when both agents are adjacent to it and load simultaneously, and the episode ends at the first collection.}
\label{fig:envs}
\end{figure*}

\subsection{Overcooked}
\label{app:overcooked}
\subsubsection{Environment and Task Setup}
\label{app:1}
We evaluate our approach in the Overcooked collaborative cooking environment~\cite{Carroll2019}, a standard benchmark for studying human–AI cooperation. Overcooked is a two-player game in which agents act as chefs collaborating to prepare and serve dishes through a sequence of interdependent sub-tasks such as picking up ingredients, chopping, plating, and serving. To promote diverse cooperative behaviors, we adopt a multi-recipe, required-coordination version of Overcooked (as seen in Fig.~\ref{fig:envs}) based on \citet{charakorn2023}.

In this setting, players must complete and serve one of a set of pre-defined recipes as quickly as possible, as opposed to repeatedly delivering a single dish in the original Overcooked environment. To prepare a recipe, each agent must first pick up the raw ingredient available on their side and place it on the chopping board. After chopping, the agent must pick up the processed ingredient and place it on the shared plate located at the middle counter. Once both required ingredients for the selected recipe are placed on the plate, the agent on the right side is responsible for picking up the completed dish and delivering it to the serving station. This sequence of sub-tasks demands coordination between the two agents to determine which recipe to pursue and how to synchronize their actions effectively. The ingredients are distributed such that each agent can access only a subset of them, so no recipe can be completed by one agent alone.

We use two instantiations of this environment, which differ in the set of recipes and, as described in Section~\ref{app:peerpool}, in the abilities of the partners the ego agent is paired with. In the \emph{homogeneous} setting, the environment features six ingredient types and three recipes, and every partner is able to prepare all three recipes. In the \emph{heterogeneous} setting, the environment features twelve ingredient types and six recipes, and partners differ in which recipes they are able to prepare. The six recipes are grouped into three pairs with mutually disjoint ingredient sets, so that no ingredient is shared between recipes of different pairs. Since the heterogeneous kitchen has to accommodate twice as many ingredients, it uses a larger map ($7\times 11$ instead of $7\times 8$) in which each sub-task takes more steps to complete; we therefore use a horizon of $80$ steps in the heterogeneous setting and $40$ steps in the homogeneous setting. All other mechanics are identical in both settings.

\paragraph{Observation Space}
Each agent receives a 105-dimensional observation vector encoding features such as agent position, orientation, held objects, objects in front of the agent, and additional flags indicating visibility to account for partial observability. In the heterogeneous setting, the observation vector is extended to encode the six additional ingredient types.
\paragraph{Action Space}
Agents act in a discrete space with six possible actions: move up, down, left, or right; interact with objects; or take no action (no-op).
\paragraph{Reward Structure}
PO-Overcooked is fully cooperative; all agents share rewards. Three types of rewards are defined:
\begin{itemize}
\item \textbf{Interaction Reward:} Each time an agent interacts with an object, all agents receive 0.5 points. Repeated interactions with the same object do not yield additional reward.
\item \textbf{Progress Reward:} Agents receive 1.0 point whenever a recipe progresses, such as when a chopped ingredient is placed on a plate, advancing the recipe state.
\item \textbf{Completion Reward:} When a dish satisfying the recipe is delivered to the delivery square, each agent receives 10 points.
\end{itemize}

\subsubsection{Peer Pool Generation}
\label{app:peerpool}
We use the peer pool generation method used by ~\citet{ma2024fast}.
In our experiments, we use it to construct skill-based level-0 peer agents, each specialized in preparing a particular recipe. Each agent is positioned on the right side of the kitchen and interacts primarily with the ingredients and sub-tasks required for its assigned recipe. In the homogeneous setting, the three recipes and their corresponding ingredients are as follows:
\begin{itemize}
\item \textbf{Lettuce $\&$ Onion Salad:} requires Lettuce and Onion.
\item \textbf{Tomato $\&$ Carrot Salad:} requires Tomato and Carrot.
\item \textbf{Potato $\&$ Broccoli Salad:} requires Potato and Broccoli.
\end{itemize}
In the heterogeneous setting, the six recipes are grouped into three ingredient-disjoint pairs, which define the three partner types:
\begin{itemize}
\item \textbf{Type 1:} Lettuce $\&$ Onion Salad; Potato $\&$ Broccoli Salad.
\item \textbf{Type 2:} Tomato $\&$ Carrot Salad; Cucumber $\&$ Pepper Salad.
\item \textbf{Type 3:} Spinach $\&$ Mushroom Salad; Pumpkin $\&$ Garlic Salad.
\end{itemize}
A skill-based peer agent focuses its actions on handling the specific ingredients of its assigned recipe, whether picking up, chopping, or delivering the completed dish. For example, a Tomato $\&$ Carrot Salad agent handles only the ingredients required for that recipe, performing sub-tasks such as picking up and chopping its ingredient, placing it on the shared plate, and delivering the completed dish when positioned on the right side of the kitchen.
During execution, each agent checks whether its current sub-task is complete. If not, it follows the shortest path to the target location and performs the required action. Once the sub-task is complete, it proceeds to the next sub-task for its assigned recipe.

The adaptive level-1 partners are obtained by training against these level-0 agents, and the two settings differ in which level-0 agents a level-1 partner is trained against. In the homogeneous setting, every level-1 partner is trained against all three level-0 agents. Each level-1 partner is therefore able to prepare all three recipes, and which recipe it pursues in a given episode depends on the interaction with the ego agent: the partner adapts to the ego agent, and successful coordination requires the ego agent to establish a common recipe with a partner that is itself adapting. In the heterogeneous setting, a level-1 partner of a given type is trained only against the two level-0 agents of its type. It is therefore able to prepare only the two recipes of its type and does not interact with the ingredients of the other four recipes. Since the ingredient sets of the three types are disjoint, an ego agent paired with such a partner cannot succeed by committing to a fixed recipe; it must infer the type of its partner from the interaction and then coordinate on one of the two recipes of that type. The heterogeneous setting thus requires the ego agent to infer the partner's abilities in addition to coordinating on a recipe.

\textbf{Evaluation Pool: }In both settings, we use separate Level-1 partner populations
for training and evaluation, generated using disjoint random seeds.
In the homogeneous setting, we use $10$ Level-1 partners for
training and $10$ held-out partners for evaluation. In the
heterogeneous setting, we use $10$ training partners and $10$
held-out partners for each of the three capability types, giving
$30$ training partners and $30$ evaluation partners in total.
We select the training and evaluation populations to have low
cross-play success when Level-1 policies from the two populations
are paired together. Held-out partners do not interact with
the ego agent during training.
\subsection{Level-Based Foraging}
\label{app:lbf}
\subsubsection{Environment and Task Setup}
To test whether the proposed training regime is specific to Overcooked, we additionally evaluate on Level-Based Foraging (LBF)~\cite{lbf}, a standard multi-agent benchmark in which agents move on a grid and collect food items (Fig.~\ref{fig:envs}). We use the forced-cooperation variant of LBF, in which a food item can only be collected if both agents are orthogonally adjacent to it and issue the load action in the same time step; movement into a cell occupied by a food item or by the other agent is blocked.

To obtain a coordination problem analogous to recipe selection in Overcooked, we assign each food item one of three types. In each episode, one item of each type is placed at a random free cell of an $8 \times 8$ grid, and the two agents start at random positions on opposite halves of the grid. The episode ends as soon as an item is collected, or after $80$ steps.  The type of each item is visible to both agents, so the only information about which item the partner intends to collect comes from the partner's movement during the episode.
\paragraph{Observation Space}
Each agent receives a 22-dimensional observation vector containing its own position, the position of the other agent, and, for each of the three items, its position, a flag indicating whether it has been collected, and a one-hot encoding of its type.
\paragraph{Action Space}
Agents act in a discrete space with six possible actions: move up, down, left, or right; load; or take no action (no-op).
\paragraph{Reward Structure}
LBF is fully cooperative; both agents share rewards. Agents receive 10 points when an item is collected, and a step cost of 0.1 is subtracted at every step.
\subsubsection{Peer Pool Generation}
Following the Overcooked setup, we construct three skill-based level-0 agents, one per item type. An agent of a given type moves along the shortest path to the nearest uncollected item of its type, issues the load action once it is adjacent to the item, and ignores items of the other two types. Level-1 partners are trained against the three level-0 agents and are therefore able to collect any of the three items, which item they collect in a given episode depending on the interaction with the ego agent, as in the homogeneous Overcooked setting. Level-1 and Level-2 agents are trained for $30$M steps, and baselines are trained against the level-0 agents for $30$M steps.

\textbf{Evaluation Pool: }We train Level-1 partners against the three Level-0 specialists
using independent random seeds. We use $10$ Level-1 partners
for training the Level-2 agent and $10$ additional partners,
trained using disjoint seeds, for evaluation. We select the training and evaluation populations to have low cross-play success when Level-1 policies from the two populations
are paired together. Held-out partners do not interact with
the ego agent during training. All agents are evaluated with the same
held-out Level-1 partners.
\subsection{LLM Agents}
\label{app:llm_setup}
We also evaluate a large language model (LLM) as the ego agent in both Overcooked settings, in order to compare a pre-trained reasoner that receives the task in natural language against agents trained with the proposed regime. We use GPT-6.1 (\texttt{gpt6\_1\_sol}~\cite{openai2026gpt61sol}) to instantiate ProAgent. The LLM is paired with the same level-0 and level-1 partners used to evaluate the learning agents. Within a round, the LLM is given, for every earlier episode of the round, a summary of the events in that episode (outcome, and the pick-ups, chops, plating and delivery of both agents with their time steps), which plays the role of the interaction history available to the learning agents. Level-1 partners receive their interaction history exactly as during their training.

\paragraph{State Description}
At every decision point the LLM receives a textual description of the current state: an ASCII rendering of the kitchen layout with all objects and both agents overlaid, a legend, the recipe table, the objects held by both agents, the contents of the plate, and the list of objects the agent can currently reach together with the cell to stand on and the direction to face.
\paragraph{Action Interface}
Following ProAgent~\cite{zhang2024proagent}, the LLM acts through a small set of high-level skills with explicit preconditions: \texttt{pickup(ingredient)}, \texttt{chop()}, \texttt{put\_on\_plate()}, \texttt{place\_on\_counter()}, \texttt{pickup\_plate()}, \texttt{deliver()}, and \texttt{wait(n)}. A skill is executed by a shortest-path controller that re-plans at every step; the LLM is queried only when a skill completes or fails, and receives a deterministic description of the failure. The LLM is asked to output a short analysis, its prediction of the partner's intention, and a plan consisting of one or more skills. We also experimented with a low-level interface in which the LLM directly outputs movement and interact actions; with this interface the models we tried frequently failed to complete even the single-agent sub-tasks, so we report results with the high-level interface.
\paragraph{Evaluation Protocol and Memory}
We use the same protocol as for the learning agents: $10$ rounds of $5$ episodes each with the same partner. At the start of a round, the partner's history is reset and the LLM's memory is cleared.  Every run is checked for failed or unparseable API calls, which would otherwise default to no-op actions. All reported runs have zero failed and zero unparseable calls.

\section{Additional Results}
\label{app:additional_results}
\subsection{LLM Agents as Ego Agents}
\label{app:llm_results}
We also examine whether an LLM can coordinate with adaptive partners without being trained against them. This tests whether the model's prior knowledge and reasoning are sufficient for the different adaptation behaviors required in our domain. We instantiate ProAgent~\citep{zhang2024proagent} with GPT-6.1(~\citet{openai2026gpt61sol}) and evaluate it as the ego agent with both non-adaptive L0 and held-out adaptive L1 partners in the two capability settings. Evaluation is conducted over $10$ rounds with $5$ episodes per round. Results with adaptive partners are reported in Table~\ref{tab:adaptive-eval-app}; implementation details are provided in Appendix~\ref{app:llm_setup}.

With non-adaptive L0 partners, the LLM agent achieves $100\%$ success in both settings. This indicates that it can use the provided interface to complete the task when the partner's recipe choice is evident from its behavior. With adaptive partners with homogeneous capabilities, it achieves an average success rate of $89\%$ with a CV of $0.06$. Here, every partner can prepare every recipe, allowing the LLM agent to take initiative by choosing a recipe that its adaptive partner can follow. In this setting, ProAgent achieves higher success than several RL baselines under their original training protocols, including Generalist ($53\%$) and LiAMx ($35\%$). However, their performance increases to $87\%$ and $92\%$, respectively, under nested training. Thus, the comparison with the LLM changes when the RL agents are trained against adaptive partners. NestRL achieves a lower average success rate than ProAgent ($82\%$ vs.\ $89\%$), but also a lower CV ($0.03$ vs.\ $0.06$), indicating less variation in its performance across partners.

With adaptive partners with heterogeneous capabilities, ProAgent's success rate drops to $33\%$, with a CV of $0.41$. Here, the partner cannot necessarily follow the agent's recipe choice. The ego agent must additionally infer which recipes the partner can prepare and adjust its own choice accordingly. In contrast, NestRL achieves $77\%$ success with a CV of $0.12$, maintaining both a higher average success rate and lower variability across partners. This difference is consistent with our behavioral analysis (Section~\ref{sec:qualitative}), where NestRL learns to probe its partner and use the response to adjust its recipe choice.

\emph{These results show that the relative performance of RL and LLM agents depends on the adaptation required by the partner population. ProAgent coordinates well without training when its partner can follow its lead, but struggles when it must infer the partner's capabilities and adapt its behavior in response. Thus, the LLM's prior knowledge does not by itself ensure the different adaptation behaviors required for coordination in both settings.}

\begin{table*}[!tb]
\centering
\small
\renewcommand{\arraystretch}{1.12}
\setlength{\tabcolsep}{3pt}
\begin{tabular}{@{}l rrrr rr@{}}
\toprule
& \multicolumn{4}{c}{\textbf{Homogeneous capabilities}}
& \multicolumn{2}{c}{\textbf{Heterogeneous capabilities}} \\
\cmidrule(lr){2-5} \cmidrule(lr){6-7}
& \multicolumn{2}{c}{Original training}
& \multicolumn{2}{c}{Nested training}
& \multicolumn{2}{c}{} \\
\cmidrule(lr){2-3} \cmidrule(lr){4-5}
\textbf{Method}
& Avg. success & CV & Avg. success & CV & Avg. success & CV \\
\midrule
\textsc{PACE}
& 0.56 (0.53--0.59) & 0.16
& \textbf{0.83} (0.80--0.86) & 0.08
& 0.38 (0.35--0.41) & 0.25 \\
\textsc{Generalist}
& 0.53 (0.51--0.55) & 0.09
& \textbf{0.87} (0.84--0.90) & 0.17
& 0.32 (0.29--0.35) & 0.30 \\
\textsc{LiAM}
& 0.50 (0.48--0.52) & 0.17
& \textbf{0.70} (0.66--0.74) & 0.20
& 0.41 (0.38--0.44) & 0.29 \\
\textsc{LiAMx}
& 0.35 (0.33--0.37) & 0.31
& \textbf{0.92} (0.90--0.94) & 0.03
& 0.28 (0.25--0.31) & 0.38 \\
\textsc{LiLi}
& 0.44 (0.43--0.46) & 0.13
& \textbf{0.68} (0.64--0.72) & 0.27
& 0.36 (0.33--0.39) & 0.31 \\
\textbf{NestRL (Ours)}
& 0.63 (0.61--0.64) & 0.13
& \textbf{0.82} (0.78--0.85) & 0.03
& 0.77 (0.72--0.81) & 0.12 \\
\midrule
\multicolumn{7}{@{}l}{\textit{Zero-shot LLM ego agent}} \\
\shortstack[l]{\textsc{ProAgent}\\(GPT-6.1)}
& \multicolumn{4}{c}{0.89 (0.86--0.91); CV $=0.06$}
& 0.33 (0.29--0.38) & 0.41 \\
\bottomrule
\end{tabular}
\caption{Success rates with held-out adaptive partners (10 rounds $\times$ 5 episodes). Under homogeneous capabilities, we compare original and nested training for each RL architecture. Under heterogeneous capabilities, RL baselines use their original training protocols, while NestRL uses nested training. ProAgent uses GPT-6.1 and is evaluated zero-shot; its homogeneous result spans both training groups because neither training regime applies. Parenthesized ranges report 95\% Wilson score confidence intervals. CV $=\sigma/\mu$ measures variability across partners, where $\sigma$ and $\mu$ are the standard deviation and mean of partner-specific success rates.}
\label{tab:adaptive-eval-app}
\end{table*}
\textbf{Statistical analysis:}
We compare ProAgent's success with non-adaptive L0 and adaptive L1 partners separately in the homogeneous setting (Fisher's exact test on pooled episodes: $p=0.06$, not significant) and the heterogeneous setting ( \textbf{$p<10^{-25}$}). We also compare its performance across the two capability settings (\textbf{$\chi^2_1=298$, $p<10^{-60}$}). Within a round, its success rate does not change across the five episodes in the homogeneous setting ($87\%$ in the first episode vs.\ $92\%$ in the fifth; Fisher's exact test, $p=0.36$), but rises from $16\%$ to $46\%$ in the heterogeneous setting ($p<10^{-4}$), i.e., ProAgent identifies the partner's capabilities only over repeated episodes with the same partner.

For comparisons with RL agents, we report the difference in average success rate (ProAgent minus RL agent, in percentage points), its 95\% Newcombe confidence interval, and a pooled $\chi^2$ test over episodes. 

\textbf{In the homogeneous setting, ProAgent is significantly better than every architecture under its original training protocol}: PACE $+33$ ($29$ to $36$; $\chi^2_1=189$), Generalist $+36$ ($32$ to $39$; $\chi^2_1=221$), LiAM $+39$ ($35$ to $42$; $\chi^2_1=255$), LiAMx $+54$ ($50$ to $57$; $\chi^2_1=490$), LiLi $+45$ ($41$ to $48$; $\chi^2_1=335$), and NestRL's architecture $+26$ ($22$ to $29$; $\chi^2_1=126$), all $p<10^{-28}$. Under nested training, the differences shrink to PACE $+6$ ($1$ to $10$; $\chi^2_1=6.9$, $p=0.008$), Generalist $+2$ ($-2$ to $6$; $\chi^2_1=0.8$, $p=0.38$), LiAM $+19$ ($14$ to $24$; $\chi^2_1=54$, $p<10^{-12}$), LiAMx $-3$ ($-7$ to $0$; $\chi^2_1=2.9$, $p=0.09$), and LiLi $+21$ ($16$ to $26$; $\chi^2_1=64$, $p<10^{-14}$). 

\textbf{In the heterogeneous setting, ProAgent is not significantly different from the baselines} under their original training protocols: PACE $-5$ ($-10$ to $1$; $\chi^2_1=2.8$, $p=0.10$), Generalist $+1$ ($-4$ to $7$; $\chi^2_1=0.2$, $p=0.65$), LiAMx $+5$ ($0$ to $11$; $\chi^2_1=3.8$, $p=0.05$), LiLi $-3$ ($-8$ to $3$; $\chi^2_1=0.9$, $p=0.33$), with the exception of LiAM ($-8$; $-13$ to $-2$; $\chi^2_1=7.2$, $p=0.007$).

In particular, the comparison between NestRL and ProAgent yields \textbf{a difference of $+7$ percentage points in favor of ProAgent in the homogeneous setting} (95\% CI: $2$ to $11$; $\chi^2_1=9.3$, $p=0.002$) and \textbf{$-44$ percentage points in favor of NestRL in the heterogeneous setting} (95\% CI: $-50$ to $-37$; $\chi^2_1=155$, $p<10^{-34}$). All tests are two-sided.
\subsection{Level-Based Foraging}
\label{app:lbf_results}
To test whether the benefit of the nested training regime is specific to Overcooked, we repeat the comparison of Section~\ref{sec:results} in the cooperative Level-Based Foraging environment described in Appendix~\ref{app:lbf}. As in Overcooked, baselines are trained under their original protocols against the Level-0 pool, while NestRL is trained against the Level-1 pool. Results are reported in Table~\ref{tab:lbf}.

Level-1 agents trained against the three Level-0 specialists achieve $93$-$95\%$ success with their training partners, indicating that they learn to coordinate successfully with the specialist pool. The evaluated ego agents also achieve similarly high success with non-adaptive L0 partners ($93$-$95\%$). However, their performance differs when paired with the held-out adaptive L1 partners.

NestRL achieves a success rate of $86\%$, compared to $58\%$ for the strongest baseline, Generalist, followed by LiLi ($50\%$), LiAM ($46\%$), and LiAMx ($38\%$). While the baselines exhibit drops of $37$--$55$ percentage points between L0 and L1 partners, NestRL's success decreases by only $8$ percentage points. Thus, although the baselines learn to coordinate with fixed partners, they struggle when the partner also adapts to their behavior.

This pattern is consistent with the Overcooked results: training against adaptive partners produces a policy that coordinates more successfully with an unseen adaptive partner. These results extend the benefit of nested training to a domain with a different action structure, where successful coordination requires simultaneous joint actions rather than the sequential contributions required in Overcooked.
\begin{table}[!h]
\centering
\small
\renewcommand{\arraystretch}{1.15}
\setlength{\tabcolsep}{5pt}
\caption{Level-Based Foraging. Success rates with the three
non-adaptive Level-0 specialists (L0) and the held-out adaptive
Level-1 partners. L1 evaluation uses $10$ partners $\times$
$10$ rounds $\times$ $5$ episodes, totaling $500$ episodes
per agent. Brackets report pooled 95\% Wilson confidence
intervals for L1 success. Gap denotes L0 minus L1 success.
The final row reports the Level-1 agents' success with their
training partners for reference.}
\label{tab:lbf}
\begin{tabular}{@{}l c c c@{}}
\toprule
\textbf{Agent} & \textbf{L0} & \textbf{Held-out L1} & \textbf{Gap} \\
\midrule
\rowcolor{gray!12}
NestRL (ours)
& 0.94 & \textbf{0.86} {\tiny [0.83, 0.89]} & 0.08 \\
\textsc{Generalist}
& 0.95 & 0.58 {\tiny [0.54, 0.62]} & 0.37 \\
\textsc{LiAM}
& 0.94 & 0.46 {\tiny [0.42, 0.50]} & 0.48 \\
\textsc{LiAMx}
& 0.93 & 0.38 {\tiny [0.34, 0.42]} & 0.55 \\
\textsc{LiLi}
& 0.95 & 0.50 {\tiny [0.46, 0.54]} & 0.45 \\
\midrule
\multicolumn{4}{@{}l@{}}{%
\shortstack[l]{%
\footnotesize L1 agents vs.\ L0:\\
\footnotesize 0.93 / 0.93 / 0.95}} \\
\bottomrule
\end{tabular}
\end{table}
\section{Statistical Analysis of Coordination with Seen and Unseen Partners}
\label{app:convention-stats}

This section supports the analysis in Section~\ref{subsec:collapse}.
For each agent, we measure success with its training (seen)
partners and held-out (unseen) partners, and define the
\emph{drop} as $\Delta=p_{\mathrm{seen}}-p_{\mathrm{unseen}}$.

\paragraph{Hypothesis tests.}
For each nested-trained method $m$, we compare its drop with
the co-trained IPPO drop. The null hypothesis is
$H_0:\Delta_{\mathrm{IPPO}}=\Delta_m$.
Under independent binomial sampling, the variance of
$\hat{\Delta}_{\mathrm{IPPO}}-\hat{\Delta}_m$ is the sum
of the four proportion variances. We use a two-sided $z$-test
with this standard error and report a 95\% Wald interval
for the difference in drops. Individual drops are reported
with 95\% Newcombe intervals.

All six comparisons reject $H_0$ with $z>25$ ($p<10^{-10}$).
Nested-trained agents exhibit drops of $4.9$--$27.7$
percentage points, compared to $81.25$ percentage points
for co-trained IPPO agents.
\begin{table}[h]
\centering
\small
\setlength{\tabcolsep}{4pt}
\renewcommand{\arraystretch}{1.10}
\resizebox{\linewidth}{!}{%
\begin{tabular}{@{}lrrrr@{}}
\toprule
\textbf{Method}
& \textbf{Drop (pp) [95\% CI]}
& \textbf{$\Delta_{\mathrm{IPPO}}-\Delta_m$ (pp) [95\% CI]}
& $z$ & $p$ \\
\midrule
IPPO (co-trained)
& 81.3 [76.2, 83.9]
& -- & -- & -- \\
\midrule
PACE
& 15.0 [12.6, 17.6]
& 66.3 [62.5, 70.0]
& 34.3 & $<10^{-10}$ \\
LiAM
& 27.7 [24.7, 30.7]
& 53.6 [49.4, 57.7]
& 25.4 & $<10^{-10}$ \\
LiAMx
& 4.9 [2.9, 7.0]
& 76.4 [72.9, 79.8]
& 43.0 & $<10^{-10}$ \\
LiLi
& 24.9 [21.6, 28.2]
& 56.4 [52.0, 60.7]
& 25.3 & $<10^{-10}$ \\
Generalist
& 10.1 [7.8, 12.5]
& 71.2 [67.5, 74.8]
& 37.9 & $<10^{-10}$ \\
\textbf{NestRL (Ours)}
& 13.5 [10.8, 16.2]
& 67.8 [63.8, 71.7]
& 33.8 & $<10^{-10}$ \\
\bottomrule
\end{tabular}}
\caption{Seen-minus-unseen drops are reported with 95\% Newcombe intervals.
Differences relative to IPPO are reported with 95\% Wald
intervals and two-sided $z$-tests. All differences are in
percentage points.}
\label{tab:convention-stats}
\end{table}
\section{Qualitative Analysis of Adaptive Behavior}
\label{app:qual}
While Section~\ref{sec:results} establishes that L2 achieves higher success rates with adaptive partners, it does not explain how this is achieved. To uncover the underlying mechanisms, we analyze the interactions between teammates.

Our analysis is grounded in the behavioral taxonomy of co-adaptation proposed by \citet{vanzoelen2021coadaptation}, which characterizes mutual adaptation in human--robot teams through interaction patterns such as waiting for a team member to act, taking initiative, following a partner's action, and signaling intent. We operationalize these patterns in our domain to identify the coordination strategies that emerge under nested training. Note that these behaviors were not hard-coded into the training procedure but emerge as a result of the training regime. Their alignment with well-established behaviors indicates that the agents are not merely optimizing reward, but exhibiting structured co-adaptation.

\paragraph{Data and measures.}
L1 agents are paired with the L0 partners of their training pool. L2 agents and baselines are paired with the held-out partners reported in Section~\ref{sec:results}). An agent \emph{commits} at the first step at which it picks up an ingredient of the recipe that is eventually plated in the episode; earlier pickups are exploratory. We define $T_{\mathrm{commit}}^{\mathrm{ego}}$ and $T_{\mathrm{commit}}^{\mathrm{partner}}$ as the commitment steps of the two agents and the \emph{commitment lead} as $\Delta t = T_{\mathrm{commit}}^{\mathrm{ego}} - T_{\mathrm{commit}}^{\mathrm{partner}}$. Episodes in which no recipe was plated carry no commitment and are excluded; $N$ denotes the remaining episodes. Leads are tested with one-sided Wilcoxon signed-rank tests ($W$ is the signed-rank statistic) with Cohen's $d = \bar{\Delta t}/s_{\Delta t}$; rates are compared with two-sided Fisher's exact tests and latencies with Mann--Whitney $U$ tests.

\subsection{Level-1 Coordination Strategy: Waiting Under Uncertainty}
When paired with a fixed Level-0 (L0) agent, L1 initially faces uncertainty over which recipe L0 will pursue and must infer the partner's intended goal from observed actions. Acting before sufficient information is available risks miscoordination. An adaptive strategy is therefore to \emph{wait}---to defer committing to a recipe until L0's behavior reveals its type. This corresponds to the \emph{waiting for team member to start acting} pattern identified in human--robot teams \citep{vanzoelen2021coadaptation}. To examine whether such a strategy emerges under nested training, we test two hypotheses:
\begin{itemize}
    \item \textbf{H1:} \emph{L1 commits to a recipe only after L0 has committed, i.e., $\Delta t > 0$.}
    \item \textbf{H2:} \emph{Episodes in which L1 waits for its partner achieve higher success rates than episodes in which it does not.}
\end{itemize}

\textbf{Results:} In the homogeneous capabilities setting ($N = 648$), L1 committed after L0 in $85.5\%$ of episodes, with a median lead of $5$ steps; a one-sided Wilcoxon signed-rank test confirmed $\Delta t > 0$ ($W = 12{,}062$, $p < .001$, $d = 0.97$). In the heterogeneous capabilities setting ($N = 536$), where L0's recipe additionally determines which of L1's two recipes is feasible, waiting is more pronounced: L1 committed after L0 in $94.0\%$ of episodes, with a median lead of $16$ steps ($W = 1{,}925$, $p < .001$, $d = 1.55$). Both results support H1: systematic waiting behavior emerges in L1 agents under nested training, and it is stronger when more must be inferred about the partner. H2 is not supported: success did not differ between episodes in which L1 waited and episodes in which it committed first (homogeneous: $87.9\%$ vs.\ $86.2\%$, $U = 24{,}062$, $p = .64$; heterogeneous: $92.6\%$ vs.\ $84.4\%$, $U = 7{,}328$, $p = .10$). Waiting is thus L1's default strategy rather than a behavior it exhibits only when it pays off within an episode; the rare episodes in which L1 commits first are those in which L0's recipe happened to coincide with L1's choice.

Therefore, we validate a key assumption of NestRL: Level-1 training produces adaptive partners that exhibit systematic waiting behavior consistent with established co-adaptation patterns of interaction.

\subsection{Level-2 Coordination Strategy: Taking Initiative}
Under the nested training regime, L2 is trained against such adaptive L1 policies. From an I-POMDP perspective, L2 must anticipate how its own actions influence L1's subsequent behavior. If L1 waits under uncertainty, the adaptive strategy for L2 is to reduce that uncertainty through informative goal-directed actions. L2 should therefore initiate, and L1 should condition its subsequent actions on that initiative, resulting in complementary leader--follower coordination. This corresponds to the \emph{initiative-taking} and \emph{following a team member's action} patterns characteristic of successful co-adaptation in human--robot teams \citep{vanzoelen2021coadaptation}. We test two hypotheses:
\begin{itemize}
    \item \textbf{H3:} \emph{L2 commits to a recipe before L1, i.e., $\Delta t < 0$.}
    \item \textbf{H4:} \emph{L1's actions are contingent on L2's: L1 acts shortly after L2 and aligns with L2's recipe.}
\end{itemize}
H3 establishes that L2 takes the lead, while H4 establishes that L1 responds to that lead.

\textbf{Results:} Against held-out L1 partners, L2 committed first in $79.6\%$ of $N = 545$ episodes in the homogeneous setting (median lead $-5$ steps; $W = 8{,}954$, $p < .001$, $d = -0.97$) and in $88.1\%$ of $N = 452$ episodes in the heterogeneous setting (median lead $-12$ steps; $W = 2{,}242$, $p < .001$, $d = -1.46$), supporting H3. To test whether L1's behavior is contingent on L2's actions, we measure the \emph{response latency}, the number of steps between L2's first pickup and L1's first pickup. If L1 responds to L2 rather than acting on a fixed schedule, we expect L1 to act shortly after L2. In the homogeneous setting L1's first pickup followed L2's after a median of $6$ steps ($n = 463$), and in episodes L2 led, L1's first pickup belonged to L2's recipe in $84\%$ of cases, confirming that L1 follows L2's lead rather than pursuing an independent goal. Success did not depend on who committed first (homogeneous: $88.7\%$ when L2 led vs.\ $83.1\%$ when it followed, $U = 20{,}388$, $p = .15$; heterogeneous: $78.7\%$ vs.\ $87.0\%$, $U = 12{,}327$, $p = .15$), so L2's initiative does not come at the cost of coordination. These results support H4.

\subsection{Level-2 Coordination Strategy in the Heterogeneous Setting: Probing}

In the homogeneous setting, every L1 can complete every recipe, so L2 can take initiative by choosing a recipe for its partner to follow. L2 picks up its first ingredient at step~$4$ (median), commits to a single recipe in $89\%$ of episodes, and L1 adopts it. In the heterogeneous setting, L1 can prepare only the recipes of its capability type, which L2 does not observe. An arbitrary recipe choice may therefore be incompatible with the partner. L2 must instead \emph{probe}: perform a reversible, informative action, observe the partner's reaction, and use that response before plating. This corresponds to the \emph{signaling intent} pattern of \citet{vanzoelen2021coadaptation}, here used to elicit task-relevant information about the partner's capabilities. We test the following hypotheses:

\begin{itemize}
    \item \textbf{H5:} \emph{L2's first pickup functions as a query: L1 reacts faster when the pickup is compatible with its type than when it is not.}
    \item \textbf{H6:} \emph{L2 acts on the reaction: after an incompatible pickup, L2 switches to an ingredient of the partner's type, and switching is necessary for success.}
    \item \textbf{H7:} \emph{L2 defers the irreversible action (plating) until L1 has revealed its type.}
\end{itemize}

We call L2's first raw-ingredient pickup its \emph{query}; the query is \emph{in-type} when the ingredient belongs to a recipe the partner can prepare. \emph{Reaction latency} is the number of steps from the query to L1's first pickup of an ingredient of its own type. A \emph{switch} is a pickup of an ingredient of the partner's type by L2 following an out-of-type query. These measures are computed over episodes with at least one L2 pickup against held-out L1 partners.

\textbf{Results:}
L2's query was in-type in $87\%$ of episodes (exact binomial test against $1/3$, $p < .001$). L1 reacted to an in-type query after a median of $12$ steps, compared to $29$ steps for an out-of-type query (one-sided Mann--Whitney $U = 3{,}335$, $p < .001$). It also reacted within $10$ steps in $32\%$ vs.\ $11\%$ of episodes (Fisher's exact $p < .001$), supporting H5: the speed of L1's reaction carries information about its type.

After an out-of-type query, L2 switched to the partner's type in $61\%$ of episodes. Switched episodes succeeded in $21\%$ of cases, while unswitched episodes never succeeded (Fisher's exact $p = .018$), supporting H6. When rolled out against a \emph{null partner} that never acts, L2 still picked up an ingredient in every episode, held it for a median of $37$ steps, and explored different capability types in $56\%$ of episodes. This shows that L2 initiates the search over types even without observing partner pickups.

Finally, L2 plated only after L1 had picked up or chopped an in-type ingredient in $97.5\%$ of episodes in which L2 plated, a median of $13$ steps after that action. Its plated ingredient was in-type in $95\%$ of episodes, and out-of-type platings never succeeded. L2 nonetheless plated before L1 in $98\%$ of episodes in which both agents plated, supporting H7: L2 waits for information from its partner before taking the irreversible step, while still taking that step first.

These results show that L2 learns to respond to its partner's adaptive behavior, and that the form of its initiative depends on what must be inferred. In the homogeneous setting, L2 leads and its partner follows; in the heterogeneous setting, L2 probes its partner and uses the response to adjust its recipe choice.

\begin{table}[h]
\centering
\caption{Commitment statistics per agent and partner set. ``Commits first'' is the fraction of episodes with $\Delta t<0$; negative median leads mean the ego commits before its partner. Wilcoxon signed-rank tests are one-sided in the direction of the sign of the median. Hom.\ and Het.\ denote the homogeneous and heterogeneous capabilities settings.}
\label{tab:commit}
\small
\begin{tabular}{llrrrrr}
\toprule
Agent (partners) & Setting & $N$ & Commits first & Median $\Delta t$ & $W$ & $d$ \\
\midrule
L1 (vs.\ L0) & Hom. & 648 & $14.5\%$ & $+5$ & $12{,}062^{***}$ & $0.97$ \\
L1 (vs.\ L0) & Het. & 536 & $6.0\%$ & $+16$ & $1{,}925^{***}$ & $1.55$ \\
L2 (vs.\ held-out L1) & Hom. & 545 & $79.6\%$ & $-5$ & $8{,}954^{***}$ & $-0.97$ \\
L2 (vs.\ held-out L1) & Het. & 452 & $88.1\%$ & $-12$ & $2{,}242^{***}$ & $-1.46$ \\
\midrule
\multicolumn{7}{l}{\emph{Original regime (trained vs.\ L0), vs.\ held-out L1}} \\
Ours (L1) & Hom. & 1{,}029 & $45.4\%$ & $+1$ & $228{,}168$ & $0.00$ \\
\textsc{Generalist} & Hom. & 867 & $23.9\%$ & $+4$ & $106{,}546^{***}$ & $0.37$ \\
\textsc{LiAM} & Hom. & 887 & $46.0\%$ & $+1$ & $178{,}137$ & $0.05$ \\
\textsc{LiAMx} & Hom. & 576 & $23.3\%$ & $+4$ & $39{,}188^{***}$ & $0.46$ \\
\textsc{LiLi} & Hom. & 826 & $37.8\%$ & $+2$ & $115{,}797^{***}$ & $0.17$ \\
\midrule
\multicolumn{7}{l}{\emph{Nested regime (trained vs.\ own L1 pool), vs.\ held-out L1}} \\
\textsc{Generalist} & Hom. & 579 & $91.7\%$ & $-6$ & $2{,}746^{***}$ & $-1.31$ \\
\textsc{LiAM} & Hom. & 469 & $73.1\%$ & $-5.5$ & $11{,}118^{***}$ & $-0.79$ \\
\textsc{LiAMx} & Hom. & 549 & $93.8\%$ & $-5$ & $888^{***}$ & $-1.35$ \\
\textsc{LiLi} & Hom. & 487 & $76.4\%$ & $-4$ & $13{,}290^{***}$ & $-0.68$ \\
\bottomrule
\multicolumn{7}{l}{$^{***}p<.001$; unmarked entries are not significant ($p>.05$).}
\end{tabular}
\end{table}

\paragraph{Summary:} The analysis reveals a coherent picture of how nested training shapes coordination behavior:
\begin{enumerate}[nosep,leftmargin=*]
    \item \textbf{Level-1 training produces adaptive partners.} L1 learns a wait-then-follow strategy: it defers commitment until the partner reveals its intent ($85.5\%$ and $94.0\%$ of episodes in the two settings), then aligns.
    \item \textbf{Level-2 training learns to respond to adaptive partners.} L2 learns to commit first in the presence of waiting partners ($79.6\%$ and $88.1\%$; $d = -0.97$ and $-1.46$), preventing mutual hesitation.
    \item \textbf{The form of initiative depends on the partner.} With homogeneous partners L2 leads and L1 follows ($84\%$ adoption, $6$-step response latency); with heterogeneous partners L2 probes with a reversible pickup, reads L1's reaction ($12$ vs.\ $29$ steps), switches when the query is incompatible, and plates only after L1 has revealed its type ($97.5\%$).
    \item \textbf{Baselines lack the critical adaptation.} Agents trained only against fixed partners rarely initiate ($23$--$46\%$ leading rate) and do not benefit when they do, explaining their degraded performance with adaptive L1 partners; the same architectures acquire the initiative strategy when trained under the nested regime ($73$--$94\%$).
\end{enumerate}

This behavioral organization - L2 initiating or probing, L1 waiting and following - aligns with the co-adaptation patterns documented in human-robot interaction \citep{vanzoelen2021coadaptation}. Nested training improves success rates by inducing the structural mechanisms of co-adaptation in the trained agents.

\section{User Study: Interface and Procedure}
We recruit $N{=}25$ participants via the Prolific crowdsourcing platform (all fluent English speakers; mean age $35.1 \pm 6.9$; 6~had no prior Overcooked experience).
A familiarization phase (one practice episode plus two 5-episode rounds with rule-based agents) ensures that participants are comfortable with the game mechanics before the main evaluation.
In the evaluation phase, each participant plays two rounds of 5~episodes with each agent.  
An episode is successful if the recipe is delivered before the 40-step horizon elapses; a round is successful if all 5~episodes result in delivery.
After each round, participants complete the NASA Task Load Index~\citep{nasa-tlx}, which measures perceived workload across six dimensions (mental demand, physical demand, temporal demand, perceived performance, effort, and frustration) on a 7-point scale.
We also ask: (i)~\emph{partner proactiveness} (``The agent proactively contributed to completing the recipe'') and
(ii)~\emph{partner adaptiveness} (``The agent adapted its behavior to coordinate with me'').
\label{app:user-study-interface}

We detail the study flow, interface design, and data collection procedure.
All screenshots are taken from the deployed study platform.

The study is hosted as a web application built on a Python~(Flask) backend that manages game sessions, serves the Overcooked environment in real time at 10~frames per second, and logs participant interactions.
Each participant receives a unique pseudonymized session identifier.
The platform was reviewed and approved by the Institutional Review Board (IRB code: SS64726).
The study proceeds through six sequential phases:
(1)~informed consent,
(2)~demographics,
(3)~game instructions,
(4)~practice,
(5)~trial rounds,
and (6)~main evaluation rounds with post-round questionnaires.
We describe each phase below.

\begin{figure}[p]
\centering
\includegraphics[width=0.85\textwidth]{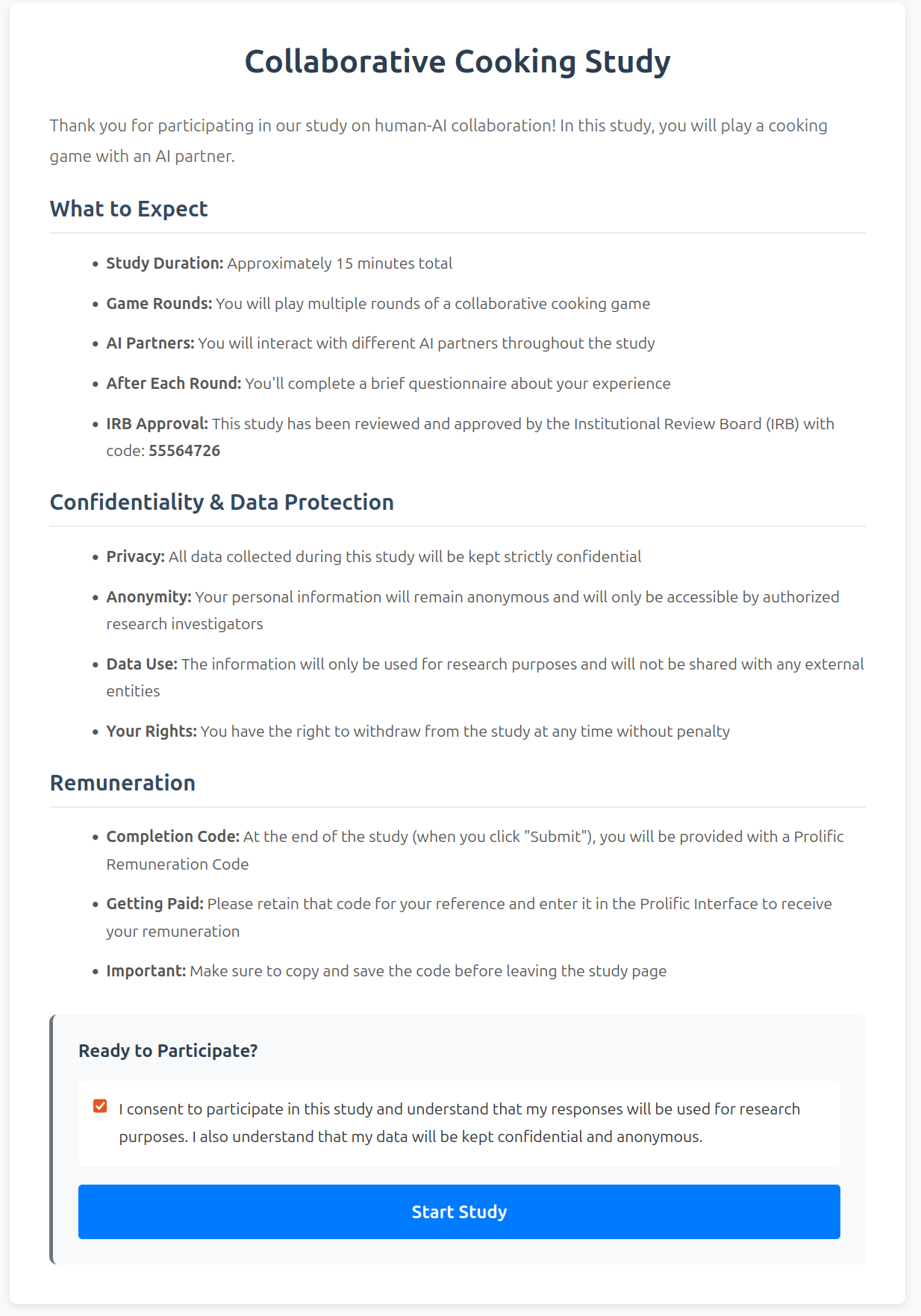}
\caption{\textbf{Phase~1: Informed consent.} Upon opening the study link, participants are briefed on the study duration (${\sim}15$~minutes), confidentiality and data protection policies, remuneration structure, and their right to withdraw at any time. The page displays the IRB approval code and explains how participants will receive their completion code and payment via the Prolific platform. A consent checkbox must be actively checked before proceeding, confirming that the participant understands the study conditions and agrees to participate.}
\label{fig:app-consent}
\end{figure}

\begin{figure}[p]
\centering
\includegraphics[width=0.55\textwidth]{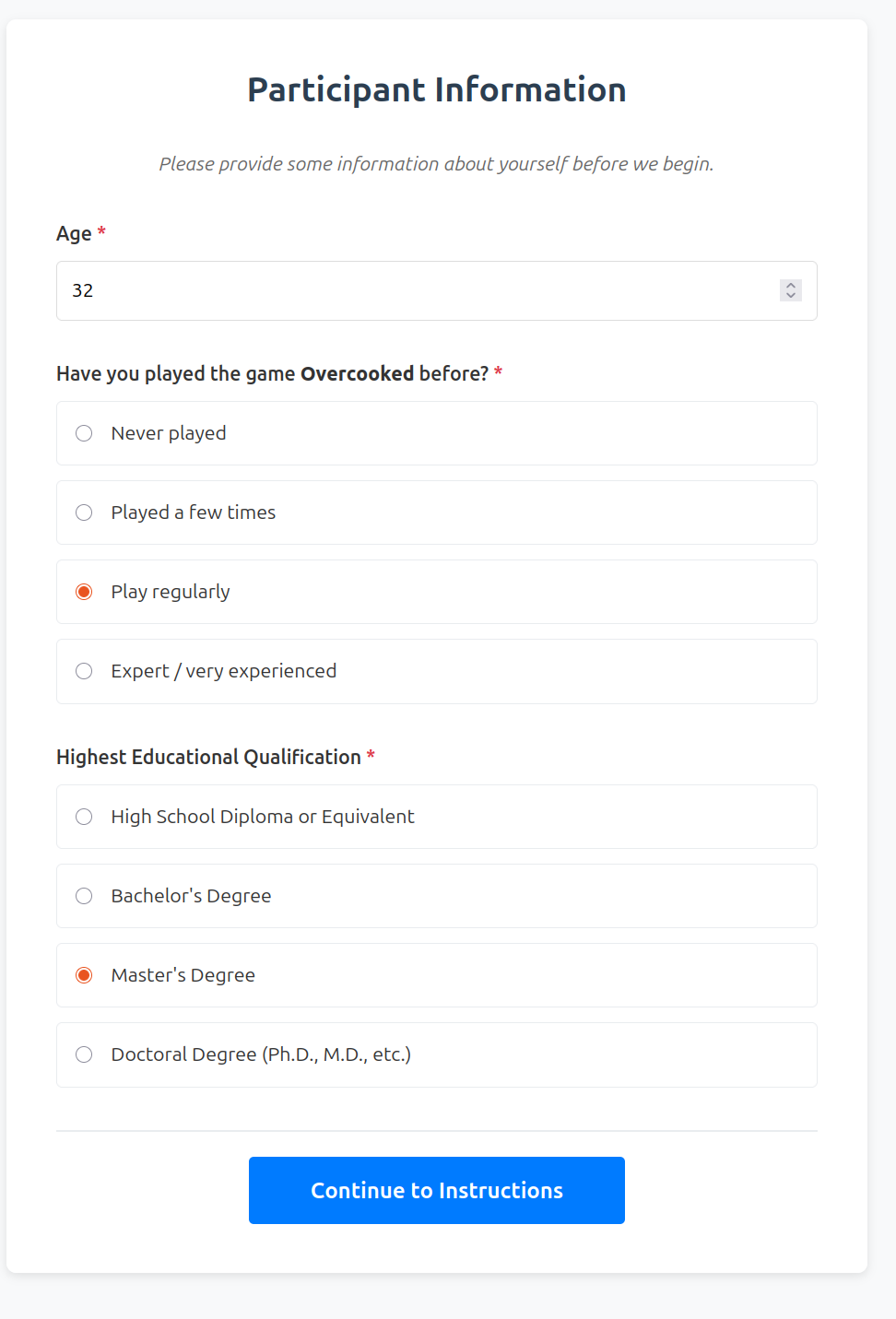}
\caption{\textbf{Phase~2: Demographics questionnaire.} Participants report their age, prior experience with the Overcooked game (on a four-point scale from ``Never played'' to ``Expert / very experienced''), and highest educational qualification. These data are used to characterize the participant pool and to check whether prior game experience moderates the observed effects.}
\label{fig:app-demographics}
\end{figure}

\begin{figure}[p]
\centering
\includegraphics[width=0.95\textwidth]{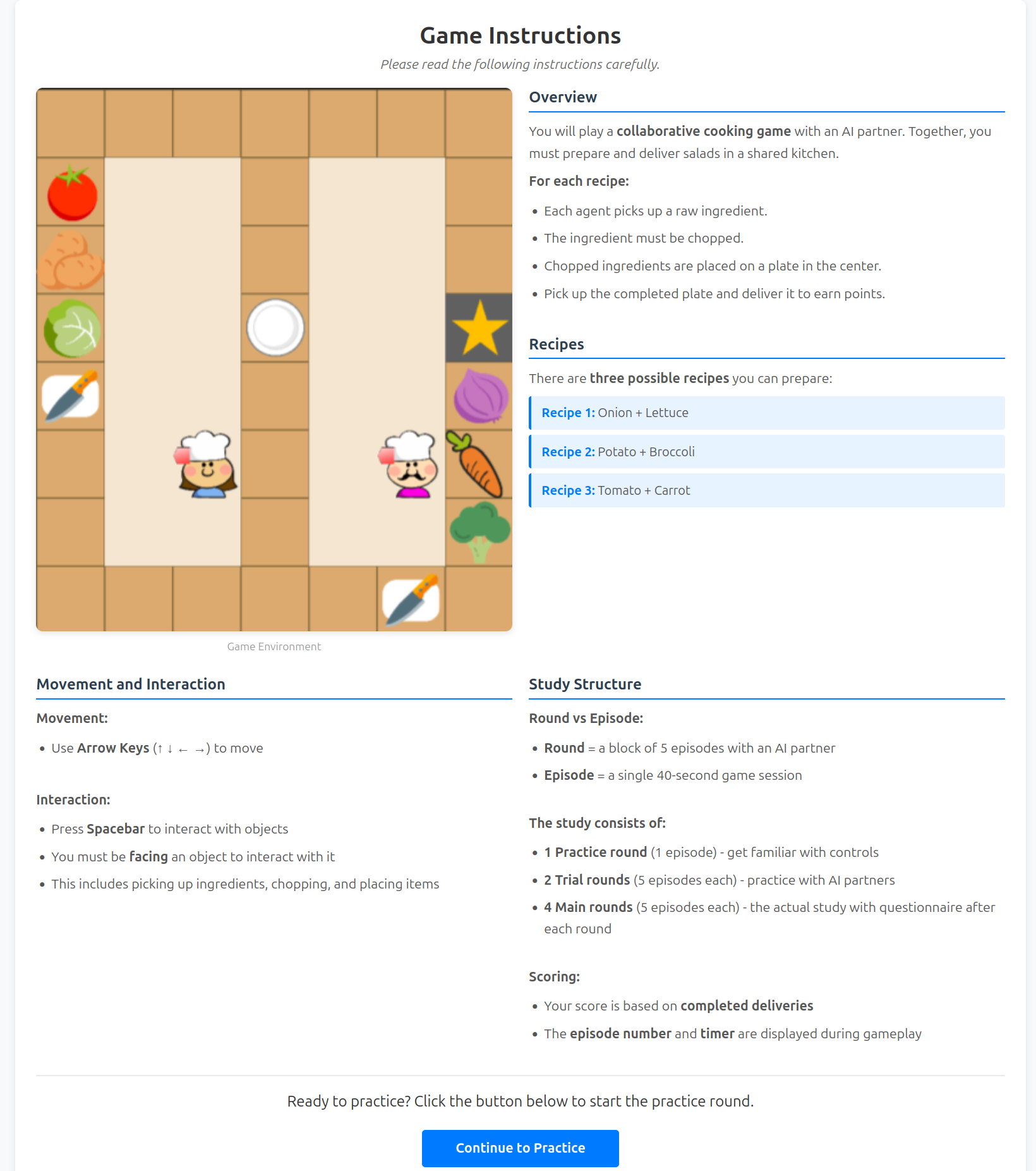}
\caption{\textbf{Phase~3: Game instructions.} The left panel displays the game environment, showing the kitchen layout with ingredient stations (tomato, lettuce, onion, potato, broccoli, carrot), chopping boards, a shared center plate, and the delivery station (star). Two chef avatars represent the human participant and the AI partner. The right panel explains the game mechanics in four steps: (1)~pick up an ingredient, (2)~chop it on a cutting board, (3)~place it on the shared plate, and (4)~deliver the completed dish. The three valid recipes are listed (Lettuce + Onion, Potato + Broccoli, Tomato + Carrot). The study structure is also explained: 1~practice round, 2~trial rounds (5~episodes each), and 4~main rounds (5~episodes each), with each episode lasting 40~seconds.}
\label{fig:app-instructions}
\end{figure}

\begin{figure}[p]
\centering
\includegraphics[width=0.95\textwidth]{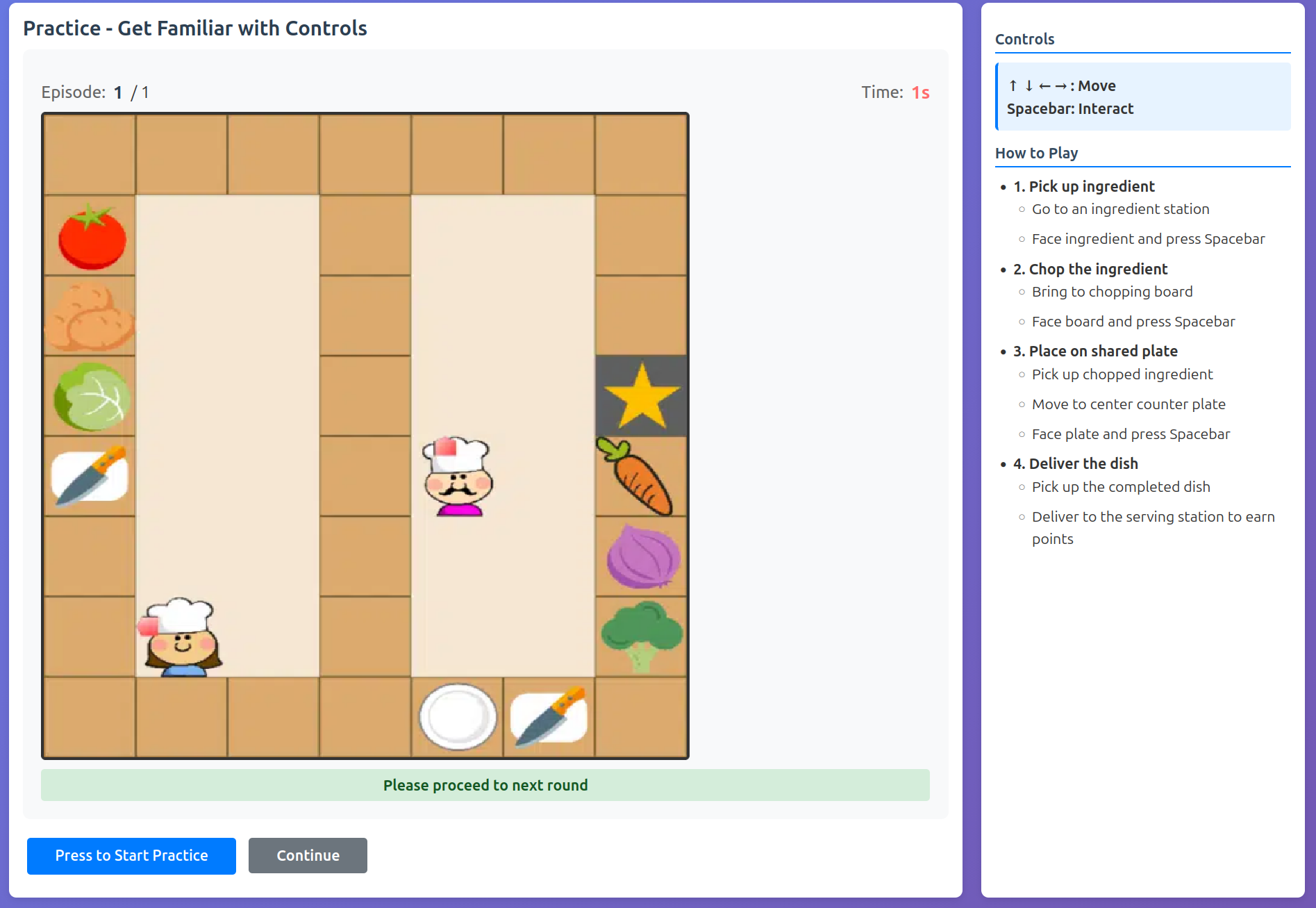}
\caption{\textbf{Phase~4: Practice round.} The participant controls one chef avatar using arrow keys (movement) and spacebar (interaction), while a rule-based AI partner controls the other. A persistent sidebar displays the control scheme and a step-by-step ``How to Play'' guide. This single-episode round is explicitly marked as practice and does not affect scoring or payment, allowing participants to build familiarity with the interface before the evaluation begins.}
\label{fig:app-practice}
\end{figure}

\begin{figure}[p]
\centering
\includegraphics[width=0.80\textwidth]{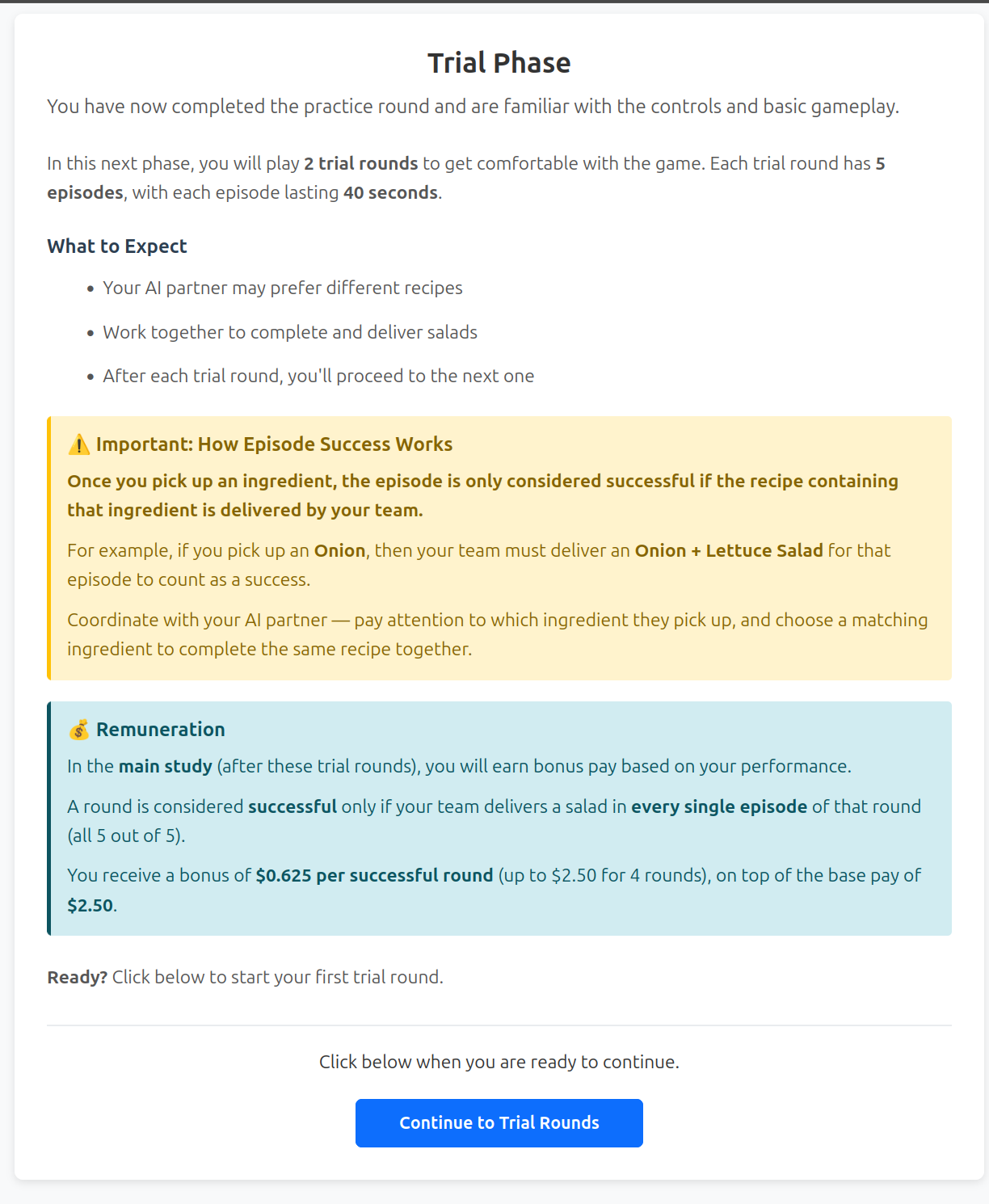}
\caption{\textbf{Phase~5a: Trial phase introduction.} Before the trial rounds begin, participants are briefed on what to expect: 2~trial rounds of 5~episodes each, with each episode lasting 40~seconds. The page explains that the AI partner may prefer different recipes and clarifies the success criterion---an episode is successful only if the recipe containing the ingredient the participant picked up is delivered. This incentivizes coordination with the AI partner. The remuneration structure for the upcoming main rounds is also previewed.}
\label{fig:app-trial-intro}
\end{figure}

\begin{figure}[p]
\centering
\includegraphics[width=0.95\textwidth]{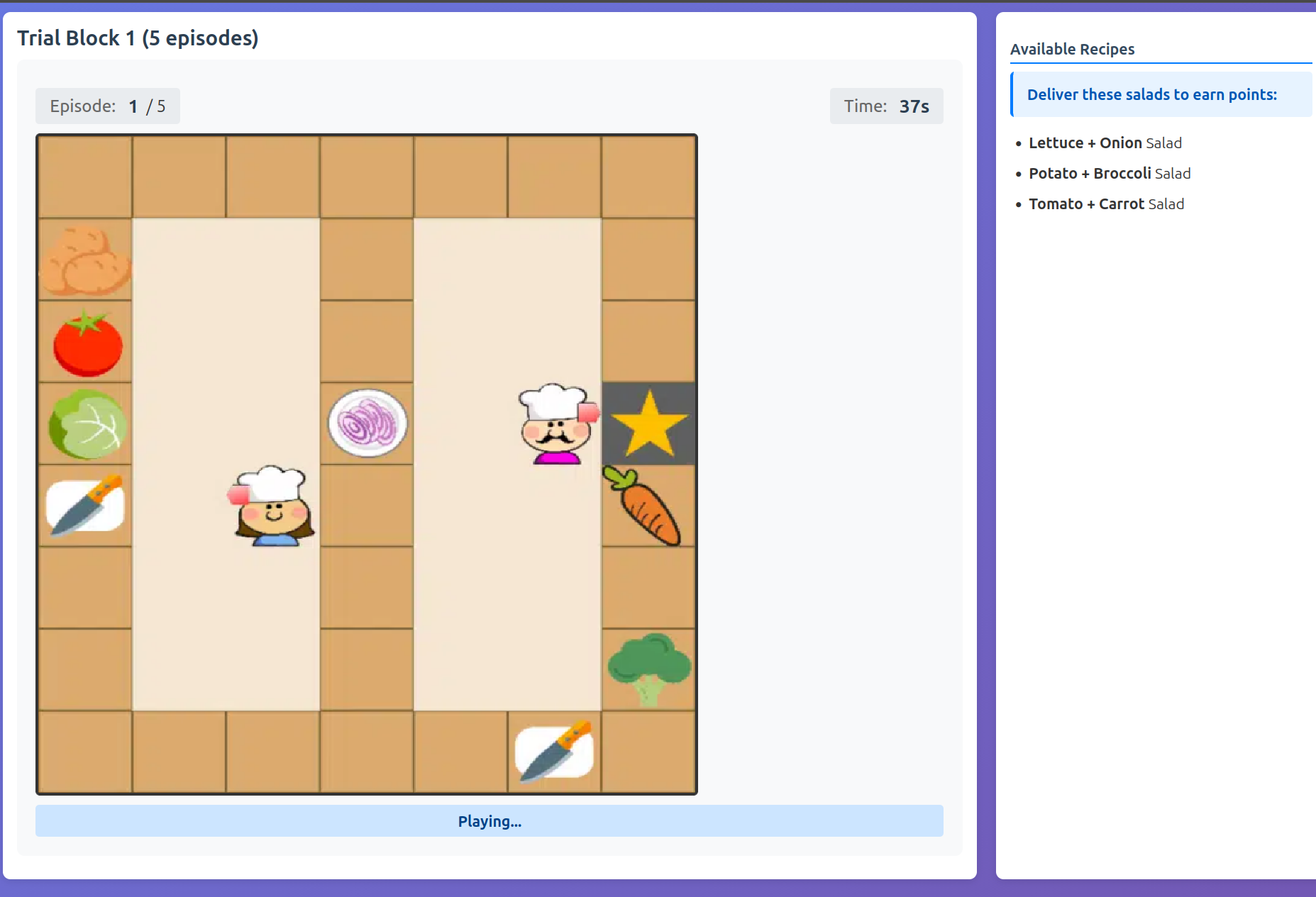}
\caption{\textbf{Phase~5b: Trial round gameplay} (Trial Block~1). The header displays the current block, episode counter (e.g., ``Episode~1/5''), and a countdown timer. The right sidebar lists the three available recipes. Participants play 2~trial blocks of 5~episodes each with rule-based AI partners, developing fluency with the interface and game dynamics before the main evaluation.}
\label{fig:app-trial-gameplay}
\end{figure}

\begin{figure}[p]
\centering
\includegraphics[width=0.70\textwidth]{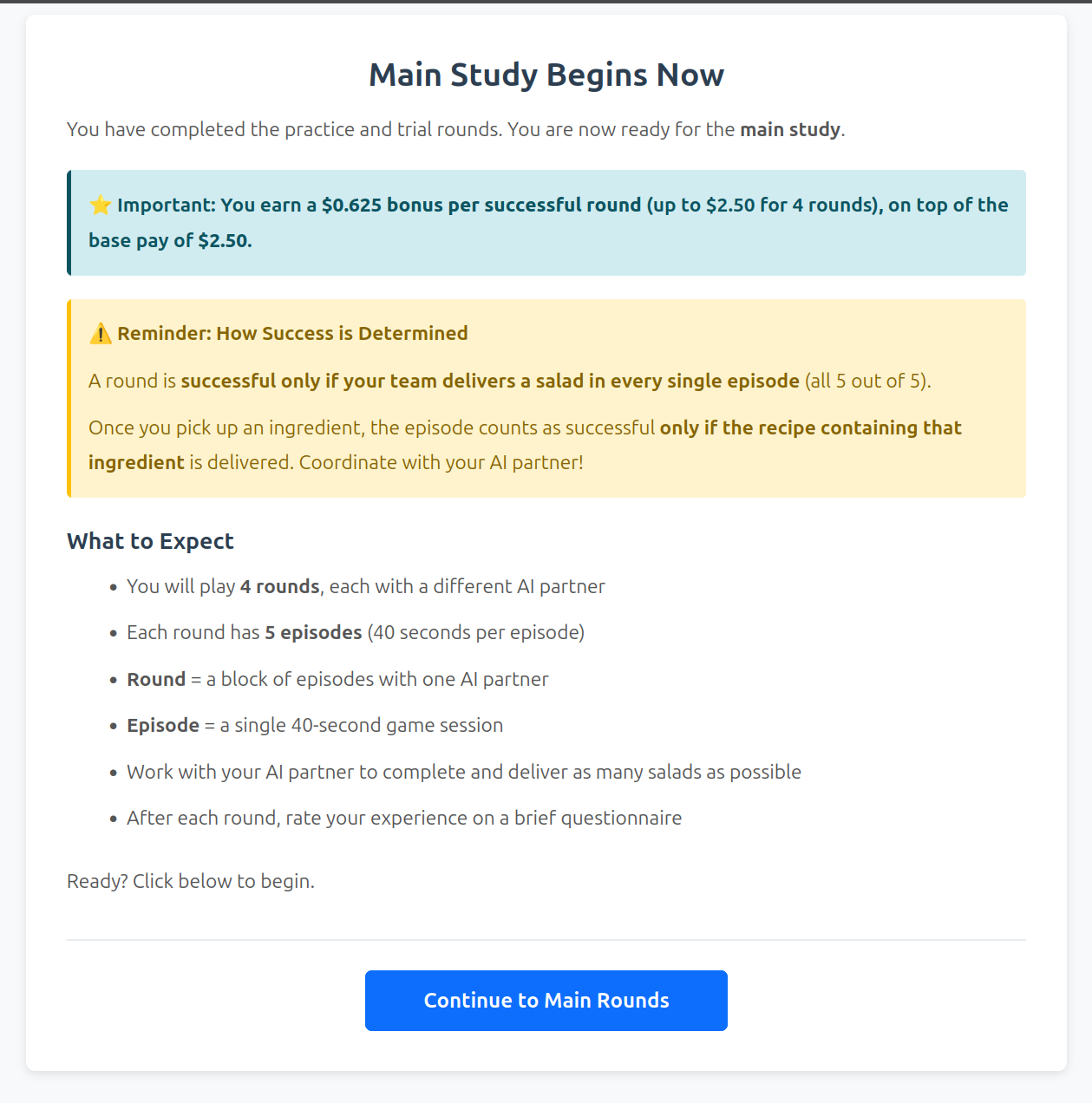}
\caption{\textbf{Phase~6a: Main study introduction.} A transition page informs participants that the performance-contingent bonus is now active: \$0.625 per successful round (all 5~episodes delivered), up to a maximum bonus of \$2.50 on top of the \$2.50 base payment. The success criterion is reiterated with a highlighted reminder. Participants will play 4~main rounds, each with a different AI partner.}
\label{fig:app-main-intro}
\end{figure}

\begin{figure}[p]
\centering
\includegraphics[width=0.95\textwidth]{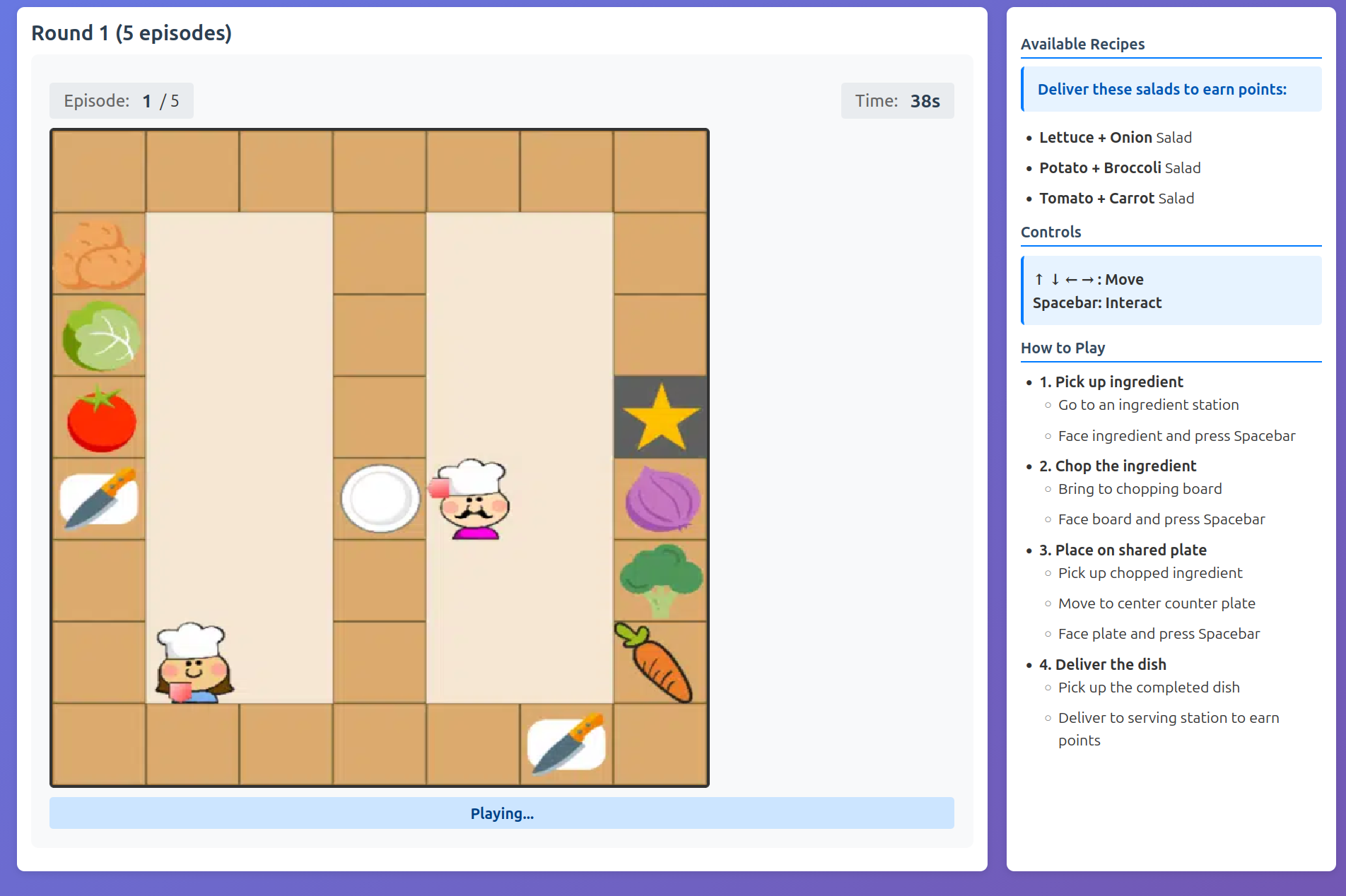}
\caption{\textbf{Phase~6b: Main round gameplay} (Round~1). The interface includes the game canvas, episode and time indicators, available recipes, control reference, and step-by-step instructions. The participant is paired with one of the evaluated AI agents (PACE or L2); each agent is played for 2~rounds. The agent assignment order is counterbalanced across participants to control for order effects.}
\label{fig:app-main-gameplay}
\end{figure}

\begin{figure}[p]
\centering
\includegraphics[width=0.60\textwidth]{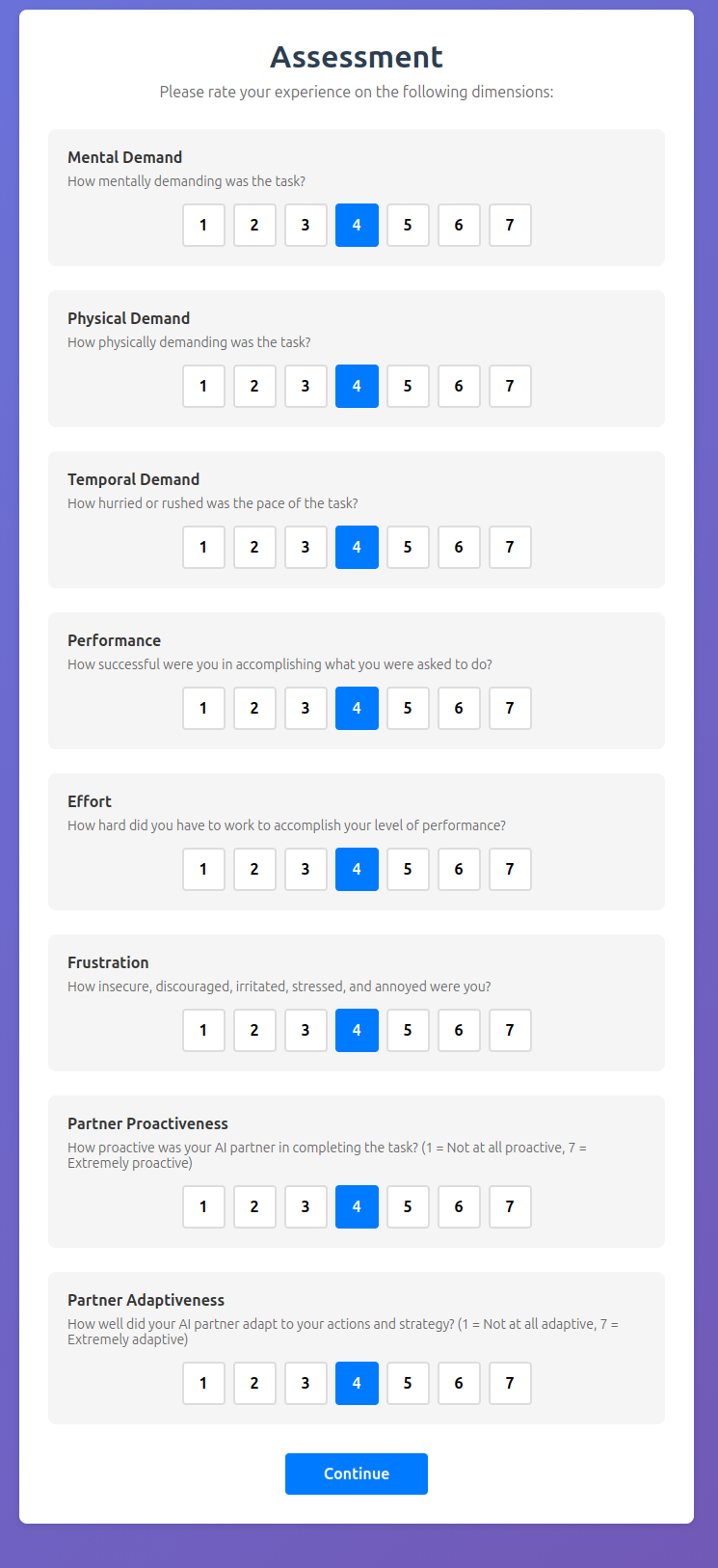}
\caption{\textbf{Phase~6c: Post-round assessment.} After each of the 4~main rounds, participants rate their experience on eight 7-point Likert scales. The first six items correspond to the NASA Task Load Index: mental demand, physical demand, temporal demand, performance, effort, and frustration. Two additional items assess perceived partner quality: \emph{proactiveness} (``How proactive was your AI partner in completing the task?'') and \emph{adaptiveness} (``How well did your AI partner adapt to your actions and strategy?''). These items test whether participants can subjectively distinguish agents trained with nested adaptation from those that are not.}
\label{fig:app-assessment}
\end{figure}
\section{User Study Results}
\label{app:user-results}
We compare NestRL against PACE in a within-subject study with
$N=25$ participants. Each participant plays two rounds of five
episodes with each agent. After each round, participants complete
the NASA-TLX questionnaire and rate the agent's proactiveness
and adaptiveness.

\begin{table}[!htb]
\centering
\small
\setlength{\tabcolsep}{5pt}
\renewcommand{\arraystretch}{1.15}
\begin{tabular}{p{7.5cm}cc}
\toprule
\textbf{Question} & \textbf{PACE} & \textbf{NestRL (ours)} \\
\midrule
How mentally demanding was the task? $^{**}\downarrow$
& 6.06 & \textbf{4.90} \\

How physically demanding was the task? $\downarrow$
& 5.18 & \textbf{4.10} \\

How rushed or hurried did you feel? $^{**}\downarrow$
& 6.64 & \textbf{4.92} \\

How successful was the team at completing the task? $^{**}\uparrow$
& 3.60 & \textbf{5.48} \\

How hard did you have to work to accomplish the task? $^{**}\downarrow$
& 6.78 & \textbf{5.60} \\

How insecure, discouraged, or irritated did you feel? $^{**}\downarrow$
& 6.24 & \textbf{4.32} \\
\midrule
The agent proactively contributed to completing the recipe. $\uparrow$
& 2.92 & \textbf{4.24} \\

The agent adapted its behavior to coordinate with me. $^{*}\uparrow$
& 2.56 & \textbf{4.42} \\
\bottomrule
\end{tabular}
\caption{Subjective ratings from the user study ($N=25$),
averaged across rounds. The first six items assess workload
and perceived team success, and the last two assess perceived
agent behavior. Arrows indicate whether lower ($\downarrow$)
or higher ($\uparrow$) ratings correspond to better outcomes.
Participants report lower workload and higher perceived team
success with NestRL, and rate it as more adaptive.
Significance is assessed using two-sided paired $t$-tests
($^{*}p<.05$, $^{**}p<.01$).}
\label{tab:user-study-tlx}
\end{table}

NestRL achieves a higher episode delivery rate than PACE
($79.6\%$ vs.\ $52.0\%$; paired $t(24)=2.41$, $p=.024$,
Cohen's $d_z=0.48$). Figure~\ref{user-study2} shows delivery
rate by episode position within a round, aggregated across
both rounds per agent. PACE's performance remains relatively
flat across episodes, while NestRL's performance improves.
A paired $t$-test on participant-level slopes shows a significant
difference between the agents
($t(24)=2.87$, $p=.008$, $d_z=0.57$).

Table~\ref{tab:user-study-tlx} reports the subjective ratings.
NestRL receives lower mean ratings for mental demand, physical
demand, temporal demand, effort, and frustration, and a higher
rating for perceived team success. Participants report
significantly lower temporal demand with NestRL
($M=4.92$ vs.\ $6.64$; $t(24)=-3.84$, $p=.001$, $d_z=-0.77$).
They also rate NestRL as significantly more adaptive
($M=4.42$ vs.\ $2.56$; $t(24)=2.53$, $p=.018$, $d_z=0.51$).
Proactiveness ratings are also higher for NestRL, although
the difference is not statistically significant
($M=4.24$ vs.\ $2.92$; $t(24)=1.89$, $p=.071$, $d_z=0.38$).
\clearpage
\newpage
\section*{NeurIPS Paper Checklist}
\begin{enumerate}

\item {\bf Claims}
    \item[] Question: Do the main claims made in the abstract and introduction accurately reflect the paper's contributions and scope?
    \item[] Answer: \answerYes{} 
    \item[] Justification: Yes we have provided the theoretical and empirical results for the claims we made in the abstract and the introduction section of the paper.

\item {\bf Limitations}
    \item[] Question: Does the paper discuss the limitations of the work performed by the authors?
    \item[] Answer: \answerYes{} 
    \item[] Justification: we have provided the limitation of the proposed framework in section 3.1. 

\item {\bf Theory assumptions and proofs}
    \item[] Question: For each theoretical result, does the paper provide the full set of assumptions and a complete (and correct) proof?
    \item[] Answer: \answerYes{} 
    \item[] Justification: Yes we have given the full set of assumptions and a complete (and correct) proof of Theorem 1 (refer to the section 3.2).

    \item {\bf Experimental result reproducibility}
    \item[] Question: Does the paper fully disclose all the information needed to reproduce the main experimental results of the paper to the extent that it affects the main claims and/or conclusions of the paper (regardless of whether the code and data are provided or not)?
    \item[] Answer: \answerYes{} 
    \item[] Justification: we have provided the full code-base in the supplementary materials and the training hyperparameters are provided in the appendix.                                                                                

\item {\bf Open access to data and code}
    \item[] Question: Does the paper provide open access to the data and code, with sufficient instructions to faithfully reproduce the main experimental results, as described in supplemental material?
    \item[] Answer: \answerYes{} 
    \item[] Justification:  we have provided the full code-base in the supplementary materials.

\item {\bf Experimental setting/details}
    \item[] Question: Does the paper specify all the training and test details (e.g., data splits, hyperparameters, how they were chosen, type of optimizer) necessary to understand the results?
    \item[] Answer: \answerYes{} 
    \item[] Justification: Yes, we have provided all the hyper-parameters detail of every baselines in the appendix B.

\item {\bf Experiment statistical significance}
    \item[] Question: Does the paper report error bars suitably and correctly defined or other appropriate information about the statistical significance of the experiments?
    \item[] Answer: \answerYes{} 
    \item[] Justification: Yes, all the experiments results are reported with appropriate statistical significance.

\item {\bf Experiments compute resources}
    \item[] Question: For each experiment, does the paper provide sufficient information on the computer resources (type of compute workers, memory, time of execution) needed to reproduce the experiments?
    \item[] Answer: \answerYes{} 
    \item[] Justification: We used 8, RTX-8000 50 GB gpus to run all the experiments.
    
\item {\bf Code of ethics}
    \item[] Question: Does the research conducted in the paper conform, in every respect, with the NeurIPS Code of Ethics \url{https://neurips.cc/public/EthicsGuidelines}?
    \item[] Answer: \answerYes{} 
    \item[] Justification: We conform with the NeurIPS Code of Ethics.

\item {\bf Broader impacts}
    \item[] Question: Does the paper discuss both potential positive societal impacts and negative societal impacts of the work performed?
    \item[] Answer: \answerNA{} 
    \item[] Justification: It is very far from deployment and we dont foresee any negative implications.
    
\item {\bf Safeguards}
    \item[] Question: Does the paper describe safeguards that have been put in place for responsible release of data or models that have a high risk for misuse (e.g., pre-trained language models, image generators, or scraped datasets)?
    \item[] Answer: \answerNA{} 
    \item[] Justification: Not applicable.

\item {\bf Licenses for existing assets}
    \item[] Question: Are the creators or original owners of assets (e.g., code, data, models), used in the paper, properly credited and are the license and terms of use explicitly mentioned and properly respected?
    \item[] Answer: \answerYes{} 
    \item[] Justification: Yes, we have cited all the baseline methods used in our paper for comparison with our proposed approach.

\item {\bf New assets}
    \item[] Question: Are new assets introduced in the paper well documented and is the documentation provided alongside the assets?
    \item[] Answer: \answerYes{} 
    \item[] Justification: We have provided the code-base of the proposed algorithm in the supplementary material and in that we have a readme.md file that includes details about training, license, limitations, etc. 

\item {\bf Crowdsourcing and research with human subjects}
    \item[] Question: For crowdsourcing experiments and research with human subjects, does the paper include the full text of instructions given to participants and screenshots, if applicable, as well as details about compensation (if any)? 
    \item[] Answer: \answerYes{} 
    \item[] Justification: Yes we have provided all the details in the appendix.

\item {\bf Institutional review board (IRB) approvals or equivalent for research with human subjects}
    \item[] Question: Does the paper describe potential risks incurred by study participants, whether such risks were disclosed to the subjects, and whether Institutional Review Board (IRB) approvals (or an equivalent approval/review based on the requirements of your country or institution) were obtained?
    \item[] Answer: \answerYes{} 
    \item[] Justification: In the appendix, we have provided the IRB certificate for the human study experiment. 

\item {\bf Declaration of LLM usage}
    \item[] Question: Does the paper describe the usage of LLMs if it is an important, original, or non-standard component of the core methods in this research? Note that if the LLM is used only for writing, editing, or formatting purposes and does \emph{not} impact the core methodology, scientific rigor, or originality of the research, declaration is not required.
    \item[] Answer: \answerNA{} 
    \item[] Justification: This paper does not involve LLMs as any important, original, or non-standard components.

\end{enumerate}

\end{document}